\documentclass{vldb}
\usepackage{cleveref}
\usepackage{algorithm}
\usepackage{times}
\usepackage{color}
\usepackage{url}
\usepackage{subfigure}
\usepackage{xspace}
\usepackage{enumerate}
\usepackage{multirow}
\usepackage{balance} 
\usepackage{url}
\usepackage{algorithmicx}

\usepackage{algpseudocode}

\newcommand{\epsilonfs}{\epsilon_{fs}}

\newcommand{\squishlist}{
   \begin{list}{$\bullet$}
    {
      \setlength{\itemsep}{0pt}
      \setlength{\parsep}{3pt}
      \setlength{\topsep}{3pt}
      \setlength{\partopsep}{0pt}
      \setlength{\leftmargin}{1.5em}
      \setlength{\labelwidth}{1em}
      \setlength{\labelsep}{0.5em} } }

\newcommand{\squishend}{
    \end{list}  }

\newcommand{\eat}[1]{}

\newtheorem{definition}{Definition}

\newtheorem{theorem}{Theorem}
\newtheorem{corollary}{Corollary}
\newtheorem{problem}{Problem}

\newcommand{\alphafs}{$\alpha$-{\sc FixedSpace}\xspace}
\newcommand{\alphamed}{$k$-{\sc RecursiveMedians}\xspace}
\begin{document}
\title{Differentially Private Algorithms for Empirical Machine Learning}
\numberofauthors{3}
\author{
            \alignauthor Ben Stoddard\\
            \affaddr{Duke University} 
                \email{stodds@cs.duke.edu}
            \alignauthor Yan Chen\\
            \affaddr{Duke University} 
                \email{yanchen@cs.duke.edu}
            \alignauthor Ashwin Machanavajjhala \\
            \affaddr{Duke University} 
                \email{ashwin@cs.duke.edu}
}
\maketitle

\begin{abstract}
An important use of private data is to build machine learning classifiers. While there is a burgeoning literature on differentially private classification algorithms, we find that they are not practical in real applications due to two reasons. First, existing differentially private classifiers provide poor accuracy on real world datasets. Second, there is no known differentially private algorithm for empirically evaluating the private classifier on a private test dataset. 

In this paper, we develop differentially private algorithms that mirror real world empirical machine learning workflows. We consider the private classifier training algorithm as a blackbox. We present private algorithms for selecting features that are input to the classifier.  Though adding a preprocessing step takes away some of the privacy budget from the actual classification process (thus potentially making it noisier and less accurate), we show that our novel preprocessing techniques signficantly increase classifier accuracy on  three real-world datasets. We also present the first private algorithms for empirically constructing receiver operating characteristic  (ROC) curves on a private test set. 

\eat{
Classifiers built on private data are employed in a number of domains
including online advertising and healthcare. Prior work has shown the
feasibility of building differentially private classification
algorithms. However, these solutions have significantly lower 
accuracy than their non-private counterparts. In this work, we
study whether preprocessing the training data using feature selection
can help improve the accuracy of off-the-shelf differentially
private classifiers.

We compare several methods in the context of a 
differentially private Naive Bayes
classifier. These include selecting
features based on naive perturbation of feature scores, taking
advantage of private clustering, 
and a novel method for differentially private 
threshold testing.  Though adding a preprocessing step takes away
some of the privacy budget from the actual classification process
(thus potentially making it noisier and less accurate), we show that our
novel preprocessing techniques increase classifier accuracy on 
three real-world datasets.
\end{abstract}
}
\end{abstract}

\section{Introduction}\label{sec:intro}

Organizations, like statistical agencies, hospitals and internet companies, collect ever increasing amounts of information from individuals with the hope of gleaning valuable insights from this data. 
The promise of effectively utilizing such `big-data' has been realized in part due to the success of off-the-shelf machine learning algorithms, especially supervised classifiers \cite{ml-book}. As the name suggests, a {\em classifier} assigns to a new observation (e.g., an individual, an email, etc.) a class chosen from a set of two or more {\em class labels} (e.g., spam/ham, healthy/sick, etc.), based on training examples with known class membership. However, when classifiers are trained on datasets containing sensitive information, ensuring that the training algorithm and the output classifier does not leak sensitive information in the data is important. For instance, Fredrikson et al  \cite{fredrikson14usenix} proposed a model inversion attack using which properties (genotype) of individuals in the training dataset can be learnt from linear regression models built on private medical data.

To address this concern, recent work has focused on developing private classifier training algorithms that ensure a strong notion of privacy called $\epsilon$-differential privacy \cite{dwork-diffpriv} -- the classifier output by a differentially private training algorithm does not significantly change due to the insertion or deletion of any one training example. Differentially private algorithms have been developed for training Naive Bayes classifiers \cite{Vaidya2013}, logistic regression \cite{Sarwate}, support vector machines \cite{DBLP:journals/corr/abs-0912-0071} and decision trees \cite{Friedman2010}. All of these techniques work by infusing noise into the training process. 

Despite the burgeoning literature in differentially private classification, their adoption by practitioners in the industry or government agencies has been startlingly rare. We believe there are two important contributing factors. First, we observe that (see experiments in Section~\ref{sec:experiments}) an off-the-shelf private classifier training algorithm, when running on real datasets, often results in classifiers with misclassification rates that are significantly higher than that output by a non-private training algorithm. 
In fact, Fredrikson et al \cite{fredrikson14usenix} also show that differentially private algorithms for the related problem of linear regression result in unacceptable error when applied to real medical datasets. 

Second, the state-of-the-art private classification algorithms do not support typical classification workflows. In particular,  real datasets usually have many features that are of little to no predictive value, and {\em feature selection} techniques \cite{Dash97featureselection} are used to identify the predictive subset of features. To date there are no differentially private feature selection techniques.  

Moreover, empirical machine learning workflows perform model diagnostics on a  test set that is disjoint from the training set. These diagnostics quantify trained classifier's prediction accuracy on unseen data. Since the unseen data must be drawn from the same distribution as the training dataset, the classifier is usually trained on a part of the dataset, and tested on the rest of the dataset. Since we assume the training dataset is private, the test dataset used for evaluating the classifier's prediction accuracy is also private. 

A typical diagnostic for measuring the prediction accuracy of a binary classifier (i.e., two classes) is the receiver operating characteristic (ROC) curve. 
Recent work \cite{privateROC} has shown that releasing an ROC curve computed on a private test set can leak sensitive information to an adversary with access to certain properties of the test dataset. Currently, there is no known method to privately evaluate the prediction accuracy of a classifier on a private test dataset.

\noindent{\bf Contributions:} 
This paper addresses the aforementioned shortcomings of the current state-of-the-art in differentially private classification. We consider the differentially private classification algorithms as a black box in order to ensure that (a) our algorithms are classifier agnostic, and (b) a privacy non-expert can use our algorithms without any knowledge of how a private classifier works. 

First, we develop a suite of  differentially private feature selection techniques based on (a) perturbing predictive scores of individual features, (b) clustering features and (c) a novel method called private threshold testing (which may be of independent interest with applications to other problems). In non-private workflows, training a classifier on a subset of predictive features  helps reduce the variance in the data and thus results in more accurate classifiers.  With multi-step differentially private workflows, either each step of the workflow  should work with a different subset of the data, or more noise must be infused in each step of the workflow. Thus it is not necessarily obvious that a workflow constituting private feature selection followed by private classification should improve accuracy over a workflow that executes a private classifier on the original data. Nevertheless, we show on real datasets and with two differentially private classifiers (Naive Bayes \cite{Vaidya2013} and logistic regression \cite{Sarwate}) that private feature selection indeed leads to significant improvement in the classifiers prediction accuracy. In certain cases, our differentially private algorithms are able to achieve accuracies attained by non-private classifiers. 

Second, we develop novel algorithms for constructing the ROC curve given a classifier and a private test set. An  ROC curve is constructed by computing the true and false positive rates on different subsets of the data. In the non-private case, typically $t$ subsets are chosen, where $t$ is the size of the test dataset. The main issue is that these subsets of the test dataset are overlapping and, hence, the true positive and false positive rates are correlated.  Thus, a naive algorithm that directly adds noise to the $t$ sets of counts results in ROC curves that are very different from the true ROC curve. We solve the first challenge by (a) carefully choosing the subsets using a differentially private recursive partitioning scheme, and (b) modeling the computation of the correlated true and false positive rates as one of privately answering a workload of one sided range queries (a problem that is well studied in the literature). Thus we can utilize algorithms from prior work (\cite{Privelet}) to accurately compute the statistics needed for the ROC curve. The utility of all our algorithms are comprehensively evaluated on three high dimensional datasets. 

\noindent{\bf Organization:} Section~\ref{sec:notation} introduces the notation. Section~\ref{sec:privatefs} describes our novel feature selection algorithms. We discuss private evaluation of classifiers in Section~\ref{sec:privateeval}. Experimental results on three real world datasets are presented in Section~\ref{sec:experiments}. Finally, we discuss related work in Section~\ref{sec:related} and conclude in Section~\ref{sec:conc}.

\section{Notation}\label{sec:notation}
Let $D$ be a dataset having $d$ attributes, and let ${\cal D}$ denote the set of all such datasets.  One of the attributes is designated as the {\em label} $L$. The remainder of the attributes are called {\em features} ${\cal F}$. We assume that all the attributes are binary (though all of the results in the paper can be extended to non-binary features). For instance, in text classification datasets (used in our experiments) binary features correspond to presence or absence of specific words  from a prespecified vocabulary. 

For any tuple $t$ in dataset $D$, let $t[L]$ denote the value of the label of the tuple, and $t[F]$ denote value of feature $F$ for that tuple. We assume that  feature vectors are sparse; every tuple has at most $s$ features with $t[F] \neq 0$. We denote by $n$ the number of tuples in $D$, and by $n_\psi$ the number of tuples in the dataset ($D_\psi$) that satisfy a boolean predicate $\psi$. For instance, $n_{F=1 \wedge L=1}$ denotes the number of tuples $t$ that satisfy $t[F]=1 \wedge t[L]=1$. We define by $P_D(\psi) = n_\psi / n$ the empirical probability of $\psi$ in the dataset $D$.

\subsection{Differential Privacy}
An algorithm satisfies differential privacy if its output on a dataset does not significantly change due to the presence or absence of any single tuple in the dataset. 
\begin{definition}[Differential Privacy \cite{dwork-diffpriv}]
Two datasets are called {\em neighbors}, denoted by $(D_1, D_2) \in N$ if either $D_1 = D_2 \cup \{t\}$ or $D_2 = D_1 \cup \{t\}$. A randomized algorithm $M$ satisfies $\epsilon$-differential privacy if $\forall s \in range(M)$ and $\forall (D_1, D_2) \in N$, 
\begin{equation}
Pr[\mathcal{M}(D) = s] \le e^\epsilon \cdot P[\mathcal{M}(D') = s]
\end{equation}
\end{definition}
Here, $\epsilon$ is the privacy parameter that controls how much an adversary can distinguish between neighboring datasets $D_1$ and $D_2$. Larger $\epsilon$ corresponds to lesser privacy and permits algorithms that retain more information about the data (i.e., utility). A variant of the above definition where neighboring datasets have the same number of tuples, but differ in one of the tuples is called {\em bounded differential privacy}.
  
Algorithms that satisfy differential privacy work by infusing noise based on a notion called {\em sensitivity}.
\begin{definition}[Global Sensitivity]
The global sensitivity of a function $f:{\cal D} \rightarrow \mathbb{R}^m$, denoted by $S(f)$ is defined as the largest L1 difference $||f(D_1) - f(D_2)||_1$, where $D_1, D_2$ are neighboring. More formally, 
\begin{equation}
S(f) \ = \ \max_{(D_1, D_2) \in N} ||f(D_1) - f(D_2)||_1
\end{equation}
\end{definition} 

A popular differentially private algorithm is the Laplace mechanism \cite{tcc:DworkMNS06} defined as follows: 
\begin{definition}\label{def:laplace}
The Laplace mechanism, denoted by $M^{Lap}$, privately computes a function $f: {\cal D} \rightarrow \mathbb{R}^m$ by computing $f(D) + \mathbf{\eta}$. $\mathbf{\eta} \in \mathbb{R}^m$ is a vector of independent random variables, where each $\eta_i$ is drawn from the Laplace distribution with parameter $S(f)/\epsilon$. That is, $P[\eta_i = z] \propto e^{-z\cdot \epsilon / S(f)}$.
\end{definition}

Differentially private algorithms satisfy the following composition properties that allow us to design complex  workflows by piecing together differentially private algorithms.
\begin{theorem}[Composition  \cite{mcsherry-mechdesign}]
Let $M_1(\cdot)$ and $M_2(\cdot)$ be $\epsilon_1$- and $\epsilon_2$-differentially private algorithms. Then,
\begin{itemize}
\item {\em Sequential Compositon:} Releasing the outputs of $M_1(D)$ and $M_2(D)$ satisfies $\epsilon_1+\epsilon_2$-differential privacy. 
\item {\em Parallel Composition:} Releasing $M_1(D_1)$ and $M_2(D_2)$, where $D_1 \cap D_2 = \emptyset$ satisfies $\max(\epsilon_1, \epsilon_2)$-differential privacy. 
\item {\em Postprocessing}: Releasing $M_1(D)$ and $M_2(M_1(D))$ satisfies $\epsilon_1$-differential privacy. That is, postprocessing an output of a differentially private algorithm does not incur any additional loss of privacy.
\end{itemize}
\end{theorem}
Hence, the privacy parameter $\epsilon$ is also called the {\em privacy budget}, and the goal is to develop differentially private workflows that maximize utility given a fixed privacy budget.

\section{Private Feature Selection}\label{sec:privatefs}
In this section, we present differentially private techniques for feature selection that improve the accuracy of differentially private classifiers. We consider the  classifer as a blackbox. 

More formally, a {\em classifier} $C$ takes as input a record of features and outputs a probability distribution over the set of labels. Throughout this paper we consider binary classifiers; i.e., $L = \{0,1\}$. Thus without loss of generality we can define the classifier as outputting a real number $p \in [0,1]$ which corresponds to the probability of $L =1$. Two examples of such classifiers include the Naive Bayes classifier and logistic regression \cite{ml-book}.

{\em Feature selection} is a dimensionality reduction technique that typically precedes classification, where only a subset of the features ${\cal F}' \subset {\cal F}$ in the dataset are retained based on some  criterion of how well ${\cal F}'$ predicts the label $L$ \cite{guyon06:feature}.  The classifier is then trained on the dataset restricted to features in ${\cal F}'$. Since features are selected based on their properties in the data, the fact that a feature is selected can allow attackers to distinguish between neighboring datasets. Thus, by sequential composition, one needs to spend a part of the total privacy budget $\epsilon$ for feature selection (say $\epsilonfs$), and use the remainder  $(\epsilon - \epsilonfs)$ for training the blackbox classifier.

Feature selection methods can be categorized as {\em filter}, {\em wrapper} and {\em embedded} methods \cite{guyon06:feature}. Filter methods assign scores to features based on their correlation with the label, and are independent of the downstream classification algorithm. Features with the best scores are retained. Wrappers are meta-algorithms that score sets of features using a statistical model. Embedded techniques include feature selection in the classification algorithm.
In this paper, we focus on {\em filter methods} so that an analyst does not need to know the internals of the private classifier. Filters compute a ranking or a score for features based on their correlation with their label. Filter methods may rank/score individual features or sets of features. We focus on methods that score individual features.  Features can be selected by  choosing the top-$k$ or those above some threshold $\tau$. 

Thus, the problem we consider can be  stated as follows: Let ${\cal D}$ be the set of all training datasets with binary features ${\cal F}$ and a binary label $L$.  Let $Q : F \times {\cal D} \rightarrow \mathrm{R}$ be a scoring function that quantifies how well $F$ predicts $L$ for a specific dataset. Let ${\cal F}_\tau$ denote the subset of features such that $\forall F \in {\cal F}_\tau, \ Q(F, D) > \tau$. Two subsets of features ${\cal F}_1, {\cal F}_2 \subset {\cal F}$, are similar if their symmetric difference is small. An example measure of similarity between ${\cal F}_1$ and ${\cal F}_2$ is the Jaccard distance (defined as $|{\cal F}_1 \cap {\cal F}_2| / |{\cal F}_1 \cup {\cal F}_2|$).
\begin{problem}[{\sc ScoreBasedFS}]
Given a dataset $D \in {\cal D}$ and a threshold $\tau$, compute ${\cal F}' \subset {\cal F}$ while satisfying $\epsilon$-differential privacy such that the similarity between ${\cal F}'$ and ${\cal F}_\tau$ is maximized.
\end{problem}

We next describe a few example scoring methods, and present differentially private algorithms for the {\sc ScoreBasedFS} problem.

\subsection{Example Scoring Functions}
\label{sec:feature-scoring}
\noindent \textbf{Total Count:}
The total count score for a feature $F$, denoted by $TC(F, D)$, is $n_{F=1}$ the number of tuples with $t[F] = 1$. Picking features with a large total count eliminates features that rarely take the value $1$.

\noindent \textbf{Difference Count:}
The difference count score for a feature $F$, denoted by $DC(F, D)$, is defined as: 
\begin{equation}
DC(F, D) \ = \ 
|n_{F=1 \wedge L=1} - n_{F=1 \wedge L=0}| 
\end{equation}
$DC(F, D)$ is large whenever one label is more frequent than the other label for $F= 1$. The difference count is smallest when both labels are equally likely for tuples with $F=1$. The difference count is the largest when $L$ is either all 1s or all 0s when conditioned on $F=1$.

\noindent \textbf{Purity Index \cite{guyon06:feature}:}
The purity index for a feature $F$, denoted by $PI(F, D)$, is defined as: 
\begin{equation}
PI(F, D) \ = \max \left\{  
\begin{array}{l}
|n_{F=1 \wedge L=1} - n_{F=1 \wedge L=0}| , \\
 |n_{F=0 \wedge L=1} - n_{F=0 \wedge L=0}|
\end{array}
\right\}
\end{equation}
$PI(F, D)$ is the same as $DC(F, D)$, except that it also considers the difference
in counts when the feature takes the value $0$.

\noindent \textbf{Information Gain:}
Information gain is a popular measure of correlation between two attributes and is defined as follows. 
\begin{definition}[Information Gain]
The information entropy $H$ of a data set $D$ is defined as:
\begin{equation}
H(D) \ = \ - \sum_{\ell \in L} P_D(L = \ell) \ln P_D(L=\ell)
\end{equation}
The information gain for a specific feature $F$ is defined as:
\begin{equation}
IG(F, D) \ =\  H(D) - \sum_{f \in F} 
P_D(F = f) \cdot H(D_{F=f})
\end{equation}
\end{definition}
Information gain of a feature $F$ is identical to the mutual information of $F$ and $L$. 

\subsection{Score Perturbation}
\label{sec:score_pert}
A simple strategy for feature selection is: (a) perturb feature scores using the Laplace mechanism, and (b) pick the features whose noisy score crosses the threshold $\tau$ (or pick the top-$k$ features sorted by noisy scores). The scale of the Laplace noise required for  privacy is $S(Q) \cdot \Delta(Q) / \epsilonfs$, where (i) $S(Q)$ is the global sensitivity of the scoring function on one feature, and (ii) $\Delta(Q)$ is the number of feature scores that are affected by adding or removing one tuple.


The sensitivity of the total score $TC$, difference score $DC$, and purity index $PI$  are all $1$. The sensitivity of information gain function has been shown to be $O(\log n)$ \cite{Friedman2010,privbayes}, where $n$ is (an upper bound on) the number of tuples in the dataset. Information gain is considered a better scoring function for feature selection in the non-private case (than $TC$, $DC$ or $PI$). However, due to its high sensitivity, feature selection based on noisy information gain results in lower accuracy, as poor features can get high noisy scores.  

Recall that $s$ is the maximum number of non-zeros appearing in any tuple. Thus, $\Delta(TC)$ and $\Delta(DC)$ are both $s$ -- these scores only change for features with a 1 in the tuple that is added or deleted. On the other hand, $IG(F, D)$ and $PI(F,D)$  can change whether $t[F]$ is 1 or 0 for the tuple that is added or deleted. Thus, $\Delta(IG)$ and $\Delta(PI)$ are equal to the total number of features $|{\cal F}| >> s$.
\footnote{If we used bounded differential privacy where neighboring datasets have the same number of tuples, we can show that $\Delta(Q) \leq 2 \cdot s$ for any scoring function, since the neighboring datasets differ in values of at most $2 \cdot s$ attributes. 
}
High sensitivity due to a large $s$ or a large $\Delta(Q)$ can result in poor utility (poor features selected). Moreover, we observe (see Section~\ref{sec:fsexp}) that a large $s$ also results in lower accuracy of private classification. We can circumvent this by \emph{sampling}; from every tuple $t$ choose at most $r$ features that have $t[F] = 1$. Sampling is able to force a bound on the number of 1s in any tuples, and thus limit the noise. However, this comes at the cost of throwing away valuable data.


\begin{algorithm}[t]
\caption{Cluster Selection $(Q(\cdot), \epsilonfs, rounds, centers, s, \tau)$ }
\begin{algorithmic}
    \State $points \gets \{\mbox{counts needed for $Q(F,D)$} \ \mid \ F \in {\cal F})\}$
    \State $clusters \gets pkmeans(points, rounds, centers, \epsilonfs, s)$
	\State $accepted \gets \{\}$
	\For{cluster in clusters}
	    \State $center \gets cluster.center()$
	    \If{$score(center) \ge \tau$}
			\State $accepted.add(cluster.features())$
		\EndIf
	\EndFor
	\State	\Return $accepted$
\end{algorithmic}
\label{alg:cluster}
\end{algorithm}
\subsection{Cluster Selection}
The shortcoming of score perturbation is that we are adding noise individually to the scores of all the features. As the number of features increases, the probability that undesirable features get chosen increases (due to high noisy scores), thus degrading the utility of the selected features. One method to reduce the amount of noise added is to {\em privately cluster} the features based on their scores, compute a representative score for each private cluster, and then pick features from high scoring clusters. This is akin to recent work on data dependent mechanisms for releasing histograms and answering range queries that group categories with similar counts and release a single noisy count for each group \cite{Kellaris-gs,DAWA,NoiseFirst}.  

We represent each feature $F$ as a vector of counts required to compute the scoring function $Q$. For instance, for $TC$ and $DC$ scoring functions, $F$ could be represented as a two dimensional point using the counts $n_{F=1 \wedge L=0}$ and $n_{F=1 \wedge L=1}$. We use private $k$-means clustering \cite{Blum-kmeans} to cluster the points.  $k$-means clustering initializes the cluster centers $(\mu_1,...,\mu_k)$ (e.g. randomly) and updates them iteratively as follows:
1) assign each point to the nearest cluster center,
2) recompute the center of each cluster,
until reaching some convergence criterion or a fixed number of iterations. This algorithm can be made to satisfy differential privacy by privately computing in each iteration (a) the number of points in each new cluster, $q_a$, and (b) the sum of the points in each cluster, $q_b$.  The global sensitivity of $q_a$ is $1$, and the global sensitivity of $q_b$ is $\Delta(Q)$ (or $r$ if sampling is used). The number of iterations is fixed, and the privacy budget is split evenly across all the iterations.

Once clusters have been privately assigned, the centers themselves can be evaluated based on their coordinates. For instance, $TC$ and $DC$ can be computed using the sum and difference (resp.) of the two-dimensional cluster centers. Depending on the score of the group all or none of the associated features will be accepted. This score does not have to be perturbed as it is computed via the centers that are the result of a private mechanism.


\begin{algorithm}[t]
\caption{Private Threshold Testing $(D, {\cal Q}, \tau)$ }
\begin{algorithmic}
	\State $\tilde{\tau} \gets \tau + Lap(1/\epsilon)$
	\For{each query $Q_i \in {\cal Q}$}
		\If{$Q_i(D) \ge \tilde{\tau}$}
			\State $v[i] \gets 1$
		\Else
			\State $v[i] \gets 0$
		\EndIf
	\EndFor
	\State	\Return $v$
\end{algorithmic}
\label{alg:ptt}
\end{algorithm}

\subsection{Private Threshold Testing}
\label{sec:ptt}
In this section we present a novel mechanism, called {\em private threshold testing} (PTT), for the {\sc ScoreBasedFS} problem whose utility is independent of both $s$ and the number of features $|{\cal F}|$, and does not require sampling. Rather than perturbing the scores of all the functions, PTT perturbs a threshold $\tau$ and returns the set of features with scores greater than the perturbed threshold. We believe PTT has applications beyond feature selection and hence we describe it more generally. 

Let ${\cal Q} = \{Q_1, Q_2, \ldots, Q_m\}$ denote a set of real valued queries over a dataset $D$, all of which have the same sensitivity $\sigma$. (In our case, each $Q_i = Q(F_i, D)$, and $m = |{\cal F}|$). PTT has as input the set of queries ${\cal Q}$ and a real number $\tau$, and outputs a vector $v \in \{0,1\}^m$, where $v[i] = 1$ if  and only if $Q_i(D) \geq \tilde{\tau}$. 

The private algorithm is outlined in Algorithm~\ref{alg:ptt}. PTT creates a noisy threshold $\tilde{\tau}$ by adding Laplace noise with scale $\sigma/\epsilon$ to $\tau$. The output vector $v$ is populated by comparing the unperturbed query answer $Q(D)$ to $\tilde{\tau}$. We can show that despite answering $m$ comparison queries (where $m$ can be very large) each with a sensitivity of $\sigma$, PTT ensures $2\sigma\epsilon$-differential privacy (rather than $m\sigma\epsilon$-differential privacy that results from a simple application of sequential composition).

\begin{theorem}\label{thm:PTT}
Private Threshold Testing is $2\sigma\epsilon$-differentially private for any set of queries ${\cal Q}$ all of which have a sensitivity $\sigma$. 
\end{theorem}
\begin{proof}(sketch)
Consider the set of queries for which PTT output 1 (call it ${\cal Q}_1$); i.e., for these queries, $Q(D) > \tilde{\tau}$. Note that for any value of the noisy threshold, say $\tilde{\tau} = z$, if $Q(D) \geq z$, then for any neighboring database $Q(D') \geq z - \sigma$ (since $\sigma$ is the sensitivity). However, since $\tilde{\tau}$ is drawn from the Laplace distribution, we have that $\frac{P(\tilde{\tau} = z)}{P(\tilde{\tau} = z - \sigma)} \leq e^{\sigma\epsilon}$.  Therefore, 
\begin{eqnarray*}
\lefteqn{\hspace{-5mm}P(Q(D) = 1, \forall Q \in {\cal Q}_1)
\ =\ \int_z P(\tilde{\tau} = z) \prod_{Q \in {\cal Q}_1}P(Q(D) > z) dz} \\
& \leq &  e^{\sigma\epsilon}\int_z P(\tilde{\tau} = z-\sigma) \prod_{Q \in {\cal Q}_1}P(Q(D') > z-\sigma) dz \\
& = & e^{\sigma\epsilon} P(Q(D') = 1, \forall Q \in {\cal Q}_1)
\end{eqnarray*}
An analogous bound for ${\cal Q}_0$ yields the requires $e^{2\sigma\epsilon}$ bound. 
\end{proof}
We can show that $\tau$ can be chosen based on the database $D$. In fact we can show the following stronger result for count-based queries. 
\begin{corollary}\label{cor:PTT}
Let ${\cal Q}$ be a set of queries with sensitivity $\sigma$. Let $\tau$ be a function on $D$ that computes the threshold, also having sensitivity $\sigma$. If the values of ${\cal Q}$ and $\tau$ on $D$ are  non-decreasing (or non-increasing) when a tuple is added (or deleted resp.) from $D$, then PTT is $\sigma\epsilon$-differentially private. 
\end{corollary}
\begin{proof}(sketch)
{\em Case (i) $\tau$ is a constant:} When $D = D' \cup \{t\}$, for all $z$, $Q_i(D) < z$ implies $Q_i(D') < z$. Thus, $r_0 = \frac{P(Q(D) = 0, \forall Q \in {\cal Q}_0)}{P(Q(D') = 0, \forall Q \in {\cal Q}_0)}$ is already bounded above by $1$, while $r_1 = \frac{P(Q(D) = 1, \forall Q \in {\cal Q}_1)}{P(Q(D') = 1, \forall Q \in {\cal Q}_1)}$ is bounded above by $e^{\sigma\epsilon}$ from proof of Thoerem~\ref{thm:PTT}. When $D' = D \cup \{t\}$, we have $r_1 < 1$ and $r_0 \leq e^{\sigma\epsilon}$. 

{\em Case (ii) $\tau$ is a function of $D$:} When $D = D' \cup \{t\}$, it holds that $P(\tilde{\tau}(D) = z)  \ \leq \ e^{\sigma\epsilon} P(\tilde{\tau}(D')  = z - \sigma)$. This is because $\tau(D')$ lies between $[\tau(D)-\sigma, \tau(D)]$. The rest of the proof remains.
\end{proof}
While Theorem~\ref{thm:PTT} applies to all our scoring functions ($TC, DC, PI$ and $IG$), the stronger result from Corollary~\ref{cor:PTT} only applies to $TC$.

\vspace{2mm}
\noindent{\bf Advantages over prior work:}
First, PTT permits releasing whether or not a set of query answers are greater than a threshold $\tau$ even if the sensitivity of releasing the answers of all the queries may be large. PTT only requires: (i) query answers to be real numbered and (ii) {\em each query} has a small sensitivity $\sigma$. The privacy guarantee is {\em independent of the number of queries}. 

Next, PTT can output whether or not a potentially unbounded number of query answers cross a threshold. This is a significant improvement over the related sparse vector technique (SVT) first described in Hardt \cite{hardt-svt}, which allows releasing upto a constant $c$ query answers that are above a threshold $\tau$. SVT works as follows: (i) pick a noisy threshold $\tilde{\tau}$ using $\epsilon/2$ privacy budget, (ii) perturb all the queries using Laplace noise using a budget of $\epsilon/2c$, and (iii) releasing the first $c$ query answers whose noisy answers are greater than $\tilde{\tau}$. Once $c$ query answers are released the algorithm halts. PTT is able to give a positive or negative answer for all queries, since it does not release the actual query answers.

Finally, PTT does not add noise to the query answers, but only compares them to a noisy threshold. This means that the answer to a query for which PTT output $1$ is in fact greater than the answer to a query for which PTT output $0$.  This is unlike {\sc noisycut}, a technique used in Lee et al \cite{lee14:kdd}. Both PTT and {\sc noisycut} solve the same problem of comparing a set of query answers to a threshold. While PTT only adds noise to the threshold, {\sc noisycut}  adds noise to both the query answers and the threshold. We experimentally show (in Section~\ref{sec:fsexp}) that PTT has better utility than {\sc noisycut}. That is, suppose ${\cal Q}_1$ is the set of queries whose true answers are $> \tau$, and ${\cal Q}^P_1$ and ${\cal Q}^N_1$ are the set of queries with a $1$ output according to PTT and {\sc noisycut}, resp. We show that ${\cal Q}^P_1$ is almost always more similar to ${\cal Q}_1$ than ${\cal Q}^N_1$. 

We quantify the utility of our feature selection algorithms by experimentally showing their effect on differentially private classifiers in Section~\ref{sec:experiments}.

\section{Private Evaluation of Classifiers}\label{sec:privateeval}
In this section, we describe an algorithm to quantify the accuracy of any binary classifier under differential privacy on a test dataset containing sensitive information.

\subsection{ROC curves}
Receiver operating characteristic (ROC) curves are typically used to quantify the accuracy of binary classifers. Let $D_{test}$ be a test dataset. For every tuple $t \in D_{test}$, let $t[L] \in \{0,1\}$ denote the true label, and $p(t) \in [0,1]$ denote the prediction returned by some classifier (probability that $t[L] = 1$). Let $n_1$ and $n_0$ denote the number of tuples with true label $1$ and $0$ respectively. 

Given a threshold $\theta$, we say that the predicted label $p_\theta(t)$ is 1 if $p(t) > \theta$. 
Based on the true label as well as the predicted label (at a given threshold $\theta$), we can quantify the accuracy of the classifier on the dataset as follows. {\em True positives}, $TP(\theta)$, are the tuples in $D_{test}$ whose true label and predicted label equals 1; i.e., $t[L] = 1 \wedge p(t) > \theta$.  {\em True negatives}, $TN(\theta)$, denote the tuples whose true and predicted labels are 0. {\em False positives}, $FP(\theta)$ are tuples whose true label is 0 but the predicted label is 1. {\em False negatives}, $FN(\theta)$ are tuples whose true label is $1$ but the predicted label is $0$. We will use the notation $TP(\theta), FP(\theta)$, etc. to both denote the set of tuples as well as the cardinality of these sets.

The true-positive rate $TPR(\theta)$ is defined as the probability that a tuple in the test set having label $1$ is correctly classified to have label $1$. 
The false-positive rate $FPR(\theta)$, is defined as the probability that a data having label $0$ is wrongly classified to have label $1$. 
Thus, 
\begin{equation}
TPR(\theta) \, = \, \frac{TP(\theta)}{n_1} \mbox{ \ \ and \ \  }
FPR(\theta) \, = \, \frac{FP(\theta)}{n_0}
\end{equation}

The Receiver operating characteristic (ROC) curve is defined by plotting  pairs of $FPR(\theta)$ versus $TPR(\theta)$ over all possible thresholds $\theta \in \Theta$. ROC curve starts at  (0,0) and ends at  (1,1).
In order to evaluate the accuracy of a binary classifier, we consider the area under the ROC curve (AUC). If the classifier is good, the ROC curve will be close to the left and upper boundary and  AUC will be close to $1$. On the other hand, if the classifier is poor, the ROC curve will be close to the $45^{\circ}$ line from (0,0) to (1,1) with AUC around $0.5$.

Recent work \cite{privateROC} has shown that releasing the actual ROC curves on a private test dataset can allow an attacker with prior knowledge to reconstruct the test dataset. An extreme yet illustrative example is as follows: suppose an attacker knows the entire test dataset except one record. Given the real ROC curve, the attacker can determine the unknown label by simply enumerating over all labels (and checking which choice led to the given ROC curve). Hence, directly releasing the real ROC curve may leak information of the data and we need a differentially private method for generating ROC curves to protect the private test dataset. 

\subsection{Private ROC curves}

There are three important challenges when generating differentially private ROC curves -- (i) how to privately compute $TPR$ and $FPR$ values, (ii) how many and what thresholds to pick, and (iii) how to ensure the monotonicity of the $TPR $ and $FPR$ values. 

One can use the Laplace mechanism to compute $TPR(\theta)$ and $FPR(\theta)$. The global sensitivity of releasing $n_0$ and $n_1$ is $1$. The global sensitivity of each of the $TP(\theta)$ and $FP(\theta)$ values equals $1$. Thus they can all released by adding Laplace noise with sensitivity $2|\Theta| + 1$, where $|\Theta|$ is the number of thresholds. However, as we will show later, the linear dependence of sensitivity on the number of thresholds can lead to significant errors in the ROC curves and the area under the curve. 

This brings us to the next concern of the number of thresholds. In the non-private case, one can pick all the prediction probabilities associated with each tuple in the test dataset as a threshold. However, as $|\Theta|$ increases, more counts need to be computed leading to more noise. Moreover, the predictions themselves cannot be publicly released, and hence the thresholds must be chosen in a private manner. 
Finally, the true $TPR$ and $FPR$ values satisfy the following monotonicity property: for all $\theta_1 \leq \theta_2$, $TPR(\theta_1) \leq TPR(\theta_2)$ and $FPR(\theta_1) \leq FPR(\theta_2)$. The private $TPR$ and $FPR$ values must also satisfy this property to get a valid ROC curve. 



Our algorithm for computing differentially private ROC curves, called {\em PriROC} (Algorithm~\ref{alg:priroc}), addresses all the aforementioned concerns.  {\em PriROC} first privately chooses a set of thresholds (using privacy parameter $\epsilon_1$). By modeling $TP$ and $FP$ values as one-sided range queries, {\em PriROC} can compute noisy $TPRs$ and $FPRs$ values (using the remaining privacy budget $\epsilon_2$) with much lower error than using the Laplace mechanism. Finally, a postprocessing step enforces the monotonicity of  $TPRs$ and $FPRs.$ We next describe these steps in detail.
\begin{algorithm}[t]
\caption{$\emph{PriROC}$ $(T, P, \epsilon)$\label{alg:priroc}}
	1. Use $\epsilon_{1}$ budget to choose the set of thresholds for computing $TPRs$ and $FPRs$\\
	2. Use $\epsilon_{2} = \epsilon - \epsilon_{1}$ budget to compute the noisy $TPRs$ and $FPRs$ at all thresholds\\
	3. Postprocess the $TPRs$ and $FPRs$ sequences to maintain consistency.
\end{algorithm}

\subsubsection{Computing noisy TPRs \& FPRs }
Suppose we are given a set of thresholds $\Theta = \{\theta_{1}, \ldots, \theta_\ell\}$, where $\theta_{i} > \theta_{i+1}$ for all $i$. Assume that $\theta_0 =1$ and $\theta_{\ell}=0$. That is, for all records $t \in D_{test}$, the prediction $p(t)$ is greater than $\theta_\ell$, but not greater than $\theta_0$. Since, $TP(\theta)$ corresponds to the number of tuples $t$ with $t[L] = 1 \wedge p(t) \geq \theta$, $TP(\theta_\ell)$ is the total number of tuples with $t[L] = 1$ (denoted by $n_1$). Similarly,  $FP(\theta_\ell)$  is the total number of tuples with $t[L] = 0$ (denoted by $n_0$). Thus: 
\begin{eqnarray}
\nonumber
TPR(\theta_{i}) = \frac{TP(\theta_{i})}{n_1} = \frac{TP(\theta_{i})}{TP(\theta_{\ell})}  \ \ \forall 1 \leq i \leq \ell\\
\nonumber
FPR(\theta_{i}) = \frac{FP(\theta_{i})}{n_0} =  \frac{FP(\theta_{i})}{FP(\theta_{\ell})} \ \ \forall 1 \leq i \leq \ell
\end{eqnarray}
Therefore, an ROC curve can be constructed by just computing $TP(\theta_i)$ and $FP(\theta_i)$ for all $\theta_i \in \Theta$.  


We next observe that the true positive and false positive counts each correspond to a set of one-sided range queries.
\begin{definition}[One-sided Range Query]\ \\
Let $X = \{x_1, x_2, \ldots, x_n\}$ denotes a set of counts. A query $q_j$ is called a {\em one sided range query}, and $q_j(X)$ is the sum of the first $j$ elements in $X$. That is, 
$q_j(X) \ = \ \sum_{i = 1}^j x_i$.
The set $C_n = \{q_1, \ldots, q_n\}$ denotes the workload of all one sided range queries.
\end{definition}

In our context, let $X^{TP}_\Theta = \{x^{TP}_1, x^{TP}_2, \ldots, x^{TP}_\ell\}$, where $x^{TP}_i$ is the number of tuples $t \in D_{test}$ with $t[L] = 1$ and $\theta_{i-1} \geq p(t) > \theta_i$. It is easy to check that $TP(\theta_i)$ is the sum of the first $i$ counts in $X^{TP}_\Theta$. We can similarly define $X^{FP}_\Theta$, and show that each $FP(\theta_i)$ is also the answer to a one-sided range query $q_i$ on $X^{TP}_\Theta$.

It is well known that the Laplace mechanism is not optimal in terms of error for the workload of one-sided range queries $C_n$. Under Laplace mechanism, each query answer would have a mean square error of $O(n^2/\epsilon^2)$.  Instead, using strategies like the hierarchical mechanism \cite{boostacc_consistency} or Privelet \cite{Privelet} allow answering each one-sided range query with no more than $O(\log^3 n/\epsilon^2)$ error. In our experiments, we use the Privelet mechanism to compute the $TP$ and $FP$ counts with a privacy budget of $\epsilon_2/2$ for each. The Privelet algorithm first computes the wavelet coefficients of the counts in $X$, adds noise to the wavelet coefficients and then reconstructs a new $\hat{X}$ from the noisy wavelet coefficients. One-sided range queries are computed on $\hat{X}^{TP}_\Theta$ to get the $TP$ counts and on $\hat{X}^{FP}_\Theta$ to get the $FP$ counts, which in turn are used to construct the noisy $TPR(\theta)$ and $FPR(\theta)$ values. Since all steps subsequent to Privelet do not use the original data, the fact that releasing $TPR(\theta)$ and $FPR(\theta)$ satisfies $\epsilon_2$-differential privacy follows from the privacy of Privelet.

\subsubsection{Choosing Thresholds}
There are two important considerations when choosing the set of thresholds $\Theta$. The number of thresholds must not be very large, as the total error is directly related to $|\Theta|$. At the same time, the thresholds must be chosen carefully so that the ROC curve on those thresholds is a good approximation of the ROC curve drawn using all the predictions in the test data. We present two heuristics for choosing $\Theta$ that take into account the above considerations. 

A simple data-independent strategy for picking the set of thresholds is to choose them uniformly from $[0,1]$. More precisely, if $n$ is the cardinality of $D_{test}$, we choose the number of thresholds to be an $\alpha \in [0,1]$ fraction of $n$, and choose the set of thresholds to be $\Theta = \{0, \frac{1}{\lfloor \alpha n \rfloor} , \frac{2}{\lfloor \alpha n \rfloor} , \dots , \frac{\lfloor \alpha n \rfloor - 1}{\lfloor \alpha n \rfloor}, 1\}$. We call this strategy \alphafs. This strategy works well when the predictions $P = \{p(t) | t \in D_{test}\}$ are uniformly spread out in $[0,1]$. Since, \alphafs is data independent, $\epsilon_1=0$, and all the privacy budget can be used for computing the $TPR$ and $FPR$ values. 

\alphafs is not a good strategy in the general case. For instance, suppose a majority of the  predictions are less than the smallest threshold $\theta_1 = \frac{1}{\lfloor \alpha n \rfloor}$. Then the ROC curve  for all those points will be approximated with a single point $(TPR(\theta_1), FPR(\theta_1))$ possibly resulting in a significant loss in accuracy in the AUC. 

Hence, we present \alphamed, a data dependent strategy that addresses skewed prediction distributions by recursively partitioning the data domain such that each partition has roughly the same number of tuples (Algorithm~\ref{alg:priquantiles}). The algorithm takes as input $\epsilon$, the privacy budget for choosing thresholds, $k$, the number of recursive steps, and $P = \{p(t) | t \in D_{test}\}$, the multiset of predictions. As the name suggests the algorithm has $k$ recursive steps, and each uses a privacy budget of $\epsilon/k$.

The algorithm recursively calls a subroutine {\sc FindMedians} computing the noisy median of all predictions within the range $(left, right)$. Initally, $left = 0$ and $right = 1$. Since median has a high global sensitivity (equal to $right$ if all values are in the range $(left, right)$), we use the smooth sensitivity framework \cite{smoothsensitivity} for computing the noisy median. We refer the reader to the original paper for details on computing the smooth sensitivity for median. We choose to sample noise from distribution  $K /(1 + |z|^2)$ (where $K$ is a normalization constant). We can generate samples from the distribution by picking $U$ uniformly from $(0,1)$ and computing $tan(\pi(U-0.5))$ (since the CDF of the distribution is $\propto arctan(z)$).

The resulting noisy median $\tilde{m}$ could fall out of the range $(left, right)$. This could either happen due to random chance, or more likely because the smooth sensitivity of the points within the range is high. A high smooth sensitivity occurs either due to a small number of data points, or when about half the data points are very close to $left$, and the rest of the points are very close to right. Then a point in the middle of the range (e.g., $(left + right)/2$) is a good partition point, and is used instead of $\tilde{m}$. The algorithm proceeds to recursively find the medians of points in $(left, \tilde{m})$ and $(\tilde{m}, right)$. 
The algorithm returns after it completes k levels of recursion. The number of thresholds output by $\alphamed$ is $2^k$.

\begin{algorithm}[t]
\caption{\alphamed}
\begin{algorithmic}
\Function{\alphamed}{$P,\epsilon,k$}
	\State $\epsilon'	 \gets \frac{\epsilon}{k}$
	\State \Return {\sc FindMedians}$(P,\epsilon',k,0,1)$
\EndFunction
\\
\Function{FindMedians}{$P,\epsilon',k, left, right$}
	\If {$k=0$}
		\Return
	\EndIf
	\State $m \gets {median}(P)$
	\State $\tilde{m} \gets m + \frac{8S^{*}_{f_{med},\epsilon'}(P)}{\epsilon_{1}}*z$, $z$ is  random noise
 $\propto \frac{1}{1+z^{2}}$
 	\If{$\tilde{m} \leq left$ or $\tilde{m} \geq right$}
 		\State $\tilde{m} = (left + right)/2$
	\EndIf
 	\State $P_{1} \gets \{ P[i] \mid P[i] \textless \tilde{m} \}$
	\State $P_{2} \gets \{ P[i] \mid P[i] \textgreater \tilde{m} \}$
	\State \Return {\sc FindMedians}$(P_{1},\epsilon',k-1,left,\tilde{m}) \cup \, \tilde{m}\, \cup$ {\sc FindMedians}$(P_{2},\epsilon',k-1,\tilde{m},right)$
\EndFunction
\end{algorithmic}
\label{alg:priquantiles}
\end{algorithm}

\begin{theorem}\label{thm:dppriquantiles}
Algorithm \ref{alg:priquantiles} (\alphamed ) satisfies $\epsilon$-differential privacy.
\end{theorem}
\begin{proof}(sketch)
The proof follows from the following statements: 
Computing the median of a set of points in each invocation of {\sc FindMedians} satisfies $\epsilon/k$-differential privacy. This is true as long as noise is drawn from the distribution $\propto 1/(1+|z|^\gamma)$, scaled appropriately by the smooth sensitivity and $\gamma \geq 1$. 
In each recursive step, computing the medians in disjoint partitions of the data satisfies $\epsilon/k$-differential privacy by parallel composition. 
Since the number of recursions is bounded by $k$, \alphamed satisfies $\epsilon$-differential privacy by serial composition.
\end{proof}

\subsubsection{Ensuring monotonicity}
$TPR(\theta)$ and $FPR(\theta)$ values in the original ROC curve are monotonic. That is, the true positive rates satisfy the following constraint: $0 \leq TPR(\theta_1) \leq \ldots \leq TPR(\theta_\ell) = 1$. However, this may not be true of the noisy $TPR$ and $FPR$ values (generated using the strategy from the previous section). We leverage the ordering constraint between the $TPR$ and $FPR$ values to boost the accuracy by using the constrained inference method proposed by Hay et al \cite{boostacc_consistency}. Since this is a postprocessing step, there is no impact on privacy.

The error introduced by our algorithms for generating ROC curves varies with different datasets. Therefore, we empirically evaluate the utility of our algorithms on real data in the next section.

\section{Experiments}\label{sec:experiments}
In this section we experimentally evaluate our differentially private algorithms for feature selection (Section~\ref{sec:fsexp})) and generating ROC curves (Section~\ref{sec:rocexp}). The main takeaways from the experimental evaluation on differentially private feature selection are: 
\squishlist
\item Spending a part of the privacy budget for private feature selection can significantly improve the misclassification rate (10\% - 15\%) of a differentially private classifier. This is despite a noisier classifier due to the smaller privacy budget. 
\item Feature selection using private threshold testing consistently results in classifiers with higher accuracy than feature selection using score perturbation and cluster selection PTT also significantly outperforms a related technique {\sc noisycut} in solving the {\sc ScoreBasedFS} problem. 
\item In the differential privacy regime, simple scoring techniques (like total count $TC$) perform as well or even better than measures like information gain $IG$ that are considered best in the non-private regime. 
\squishend
The main takeaways from the experimental evaluation on differentially private ROC curves are: 
\squishlist
\item The area under the curve (AUC) measure for the differentially private ROC curves are close to the AUC measures for the true ROC curves. Therefore, with high probability differentially private ROC curves can be used to distinguish between classifiers that are significantly different. 
\item The AUC error for ROC curves generated by PriROC is significantly smaller than AUC error for ROC curves based on true and false positive rates computed using the Laplace mechanism.
\item The \alphamed method to pick thresholds results in better ROC curves than using \alphafs. 
\item The number of thresholds chosen to generate the differentially private ROC curve does not significantly affect the AUC error.
\squishend

\subsection{Feature Selection}\label{sec:fsexp}
\newcommand{\twitter}{{\sc Twitter}\xspace}
\newcommand{\sms}{{\sc SMS}\xspace}
\newcommand{\reuters}{{\sc Reuters}\xspace}
\subsubsection{Setup}
We use three text classification datasets - \twitter, \sms and \reuters. The \twitter dataset \cite{go2009twitter} was collected for the task of sentiment classification. Each tweet is associated with a binary sentiment label -- positive or negative. The datast contains 1.6 million tweets from which we randomly sampled  7304 tweets for our experiments. We constructed binary features for every word (excluding stop words) resulting a total of 32935 features. Since each tweet contained at most 20 non-stop words, we set $s = 20$.  
The \sms dataset \cite{sms} contains  5574 SMS messages associated with spam/ham label. The dataset has a total of 8021 features. Since SMS messages are short, we again set $s = 20$.  The \reuters dataset consists of 21578 news articles tagged with topics. To get a training dataset with a binary class label, we chose a corpus of 6906 articles labeled as earnings-related or not (based on the "earn" topic keyword). Since an article does not have a word limit, we do not have a small bound on $s$ like in \twitter or \sms. The total number of features is 33389.

We choose to evaluate our feature selection algorithm on two state of the art differentially private classifiers -- Naive Bayes \cite{Vaidya2013}, and the differentially private ERM implementation of logistic regression \cite{Sarwate}. The Naive Bayes (NB) classifier assumes that the features are conditionally independent given the label $L$. Given a feature vector $x \in \{0,1\}^{|{\cal F}|}$, the predicting label given by 
\[argmax_{\ell \in \{0,1\}} Pr[L = \ell] \cdot \prod_{F \in {\cal F}} Pr[F = x[F] | L = \ell]\]
Thus, the Naive Bayes classifier can be made private by releasing differentially private counts of
$n$, $n_{L = \ell}$, and $n_{F = i \wedge L = \ell}$ from the training data set, for $i, \ell \in \{0,1\}$.

Logistic regression models the log odds of the prediction as linear function of the features. Empirical risk minimization is used to fit the linear model given a dataset. For non-private logistic regression, we have
used the prepackaged Scikit-learn logistic regression classifier \cite{scikit-learn}. We use an implementation of Chaudhuri et al's \cite{Sarwate} differentially private empirical risk minimization (henceforth called ERM) for logistic regression.

The accuracy of a classifier is measured using the fraction of predictions that match the true label on a held out test set. The results are average over 10 runs (using 10-fold cross validation) to account for the noise introduced due to differential privacy.

\subsubsection{Feature Selection Results}

\begin{figure*}[t]
  \subfigure[Twitter Non-private]{\label{fig:twit_np}\includegraphics[width=0.31\textwidth]{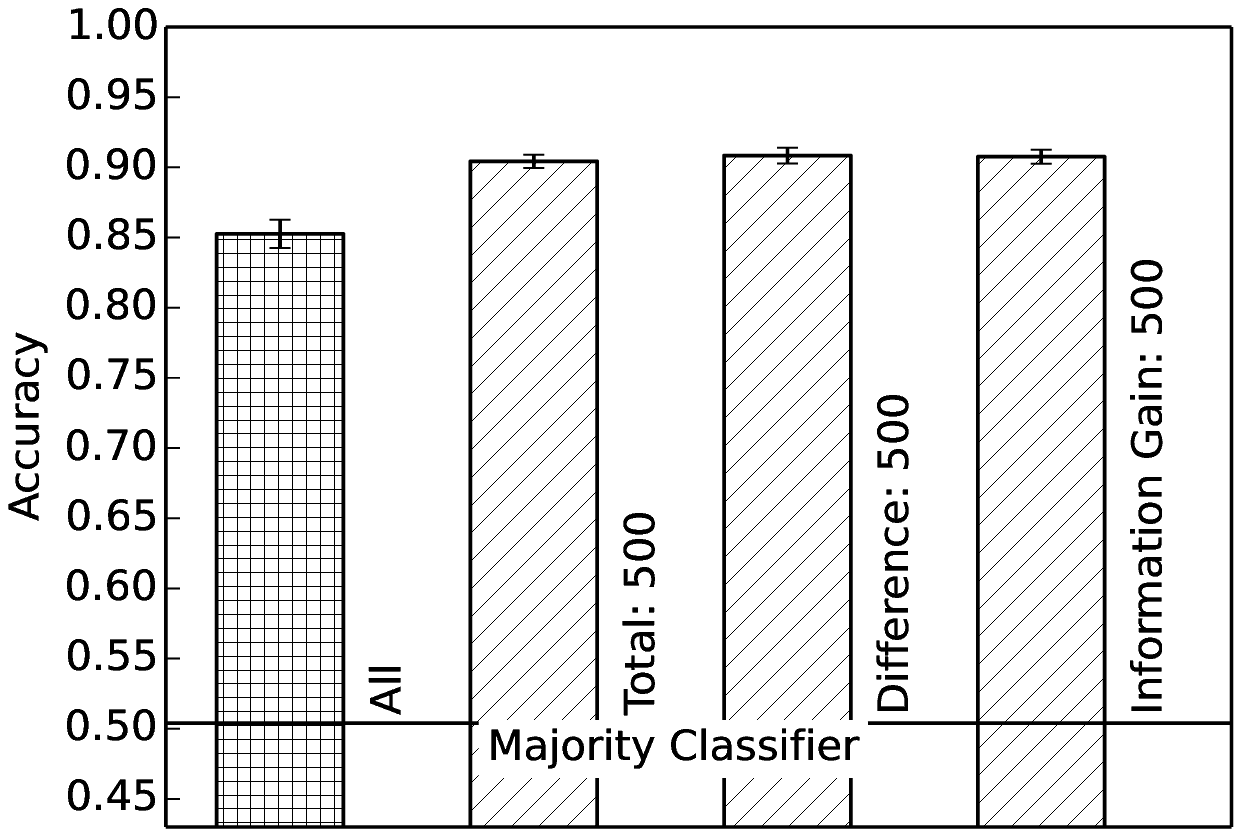}}
  \subfigure[Twitter Private]{\label{fig:twit_p}\includegraphics[width=0.69\textwidth]{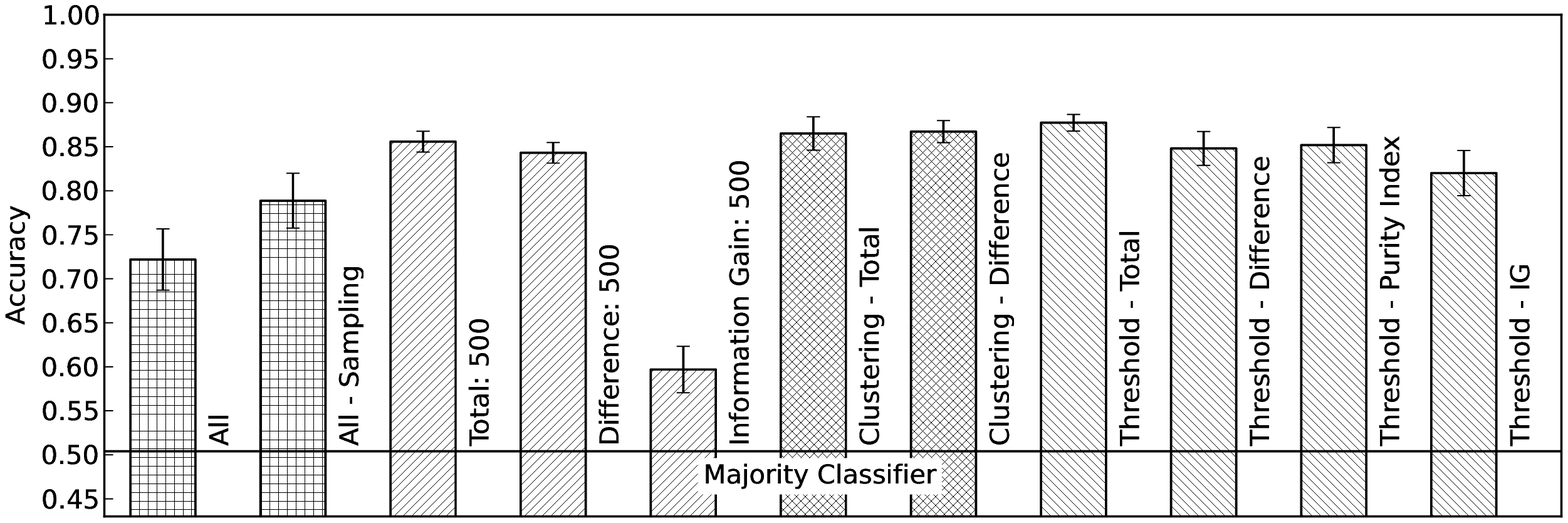}}
  \subfigure[SMS Non-private]{\label{fig:sms_np}\includegraphics[width=0.31\textwidth]{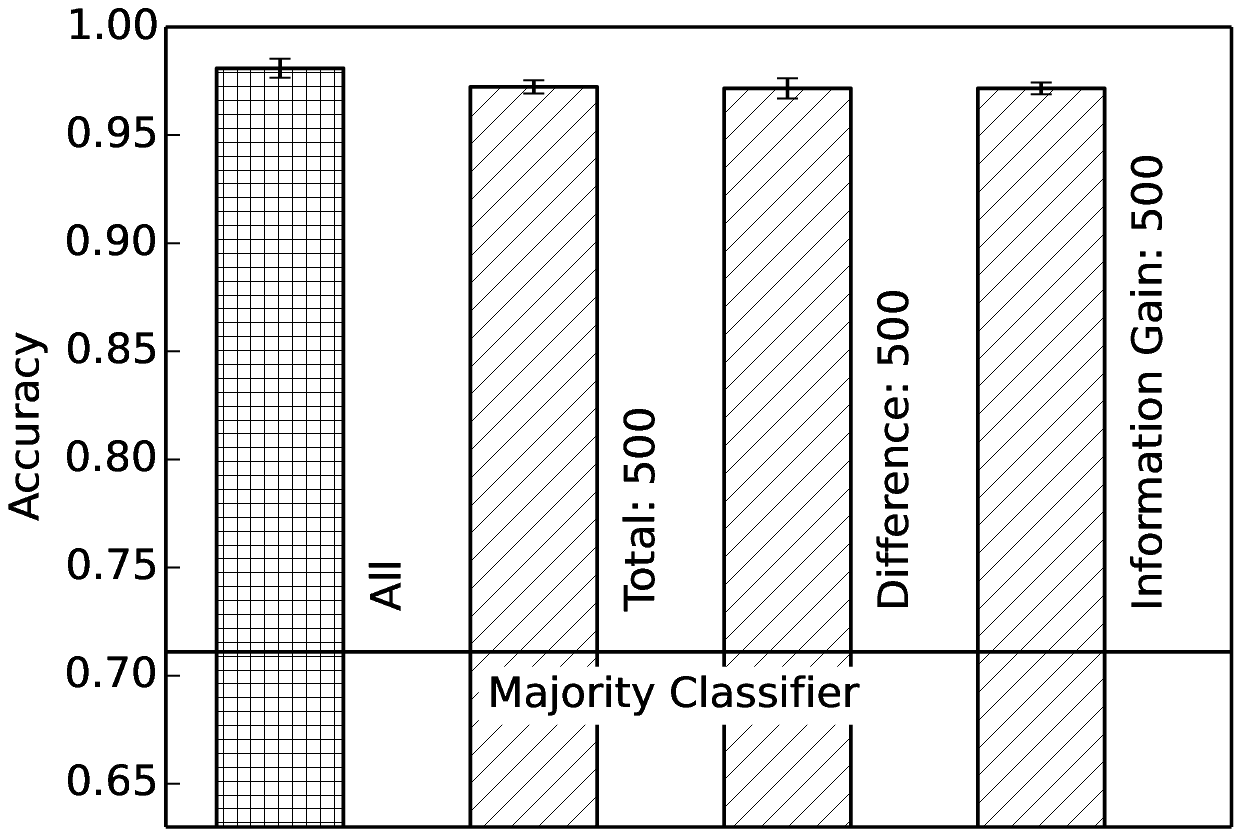}}
  \subfigure[SMS Private]{\label{fig:sms_p}\includegraphics[width=0.69\textwidth]{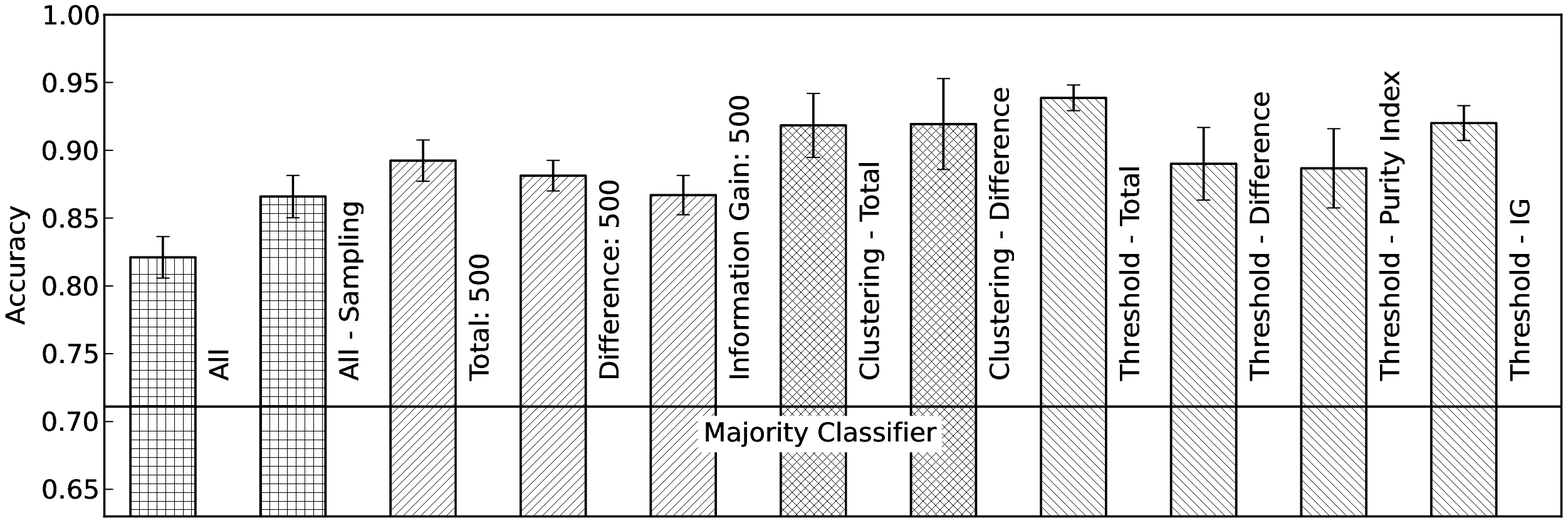}}
  \subfigure[Reuters Non-private]{\label{fig:reut_np}\includegraphics[width=0.31\textwidth]{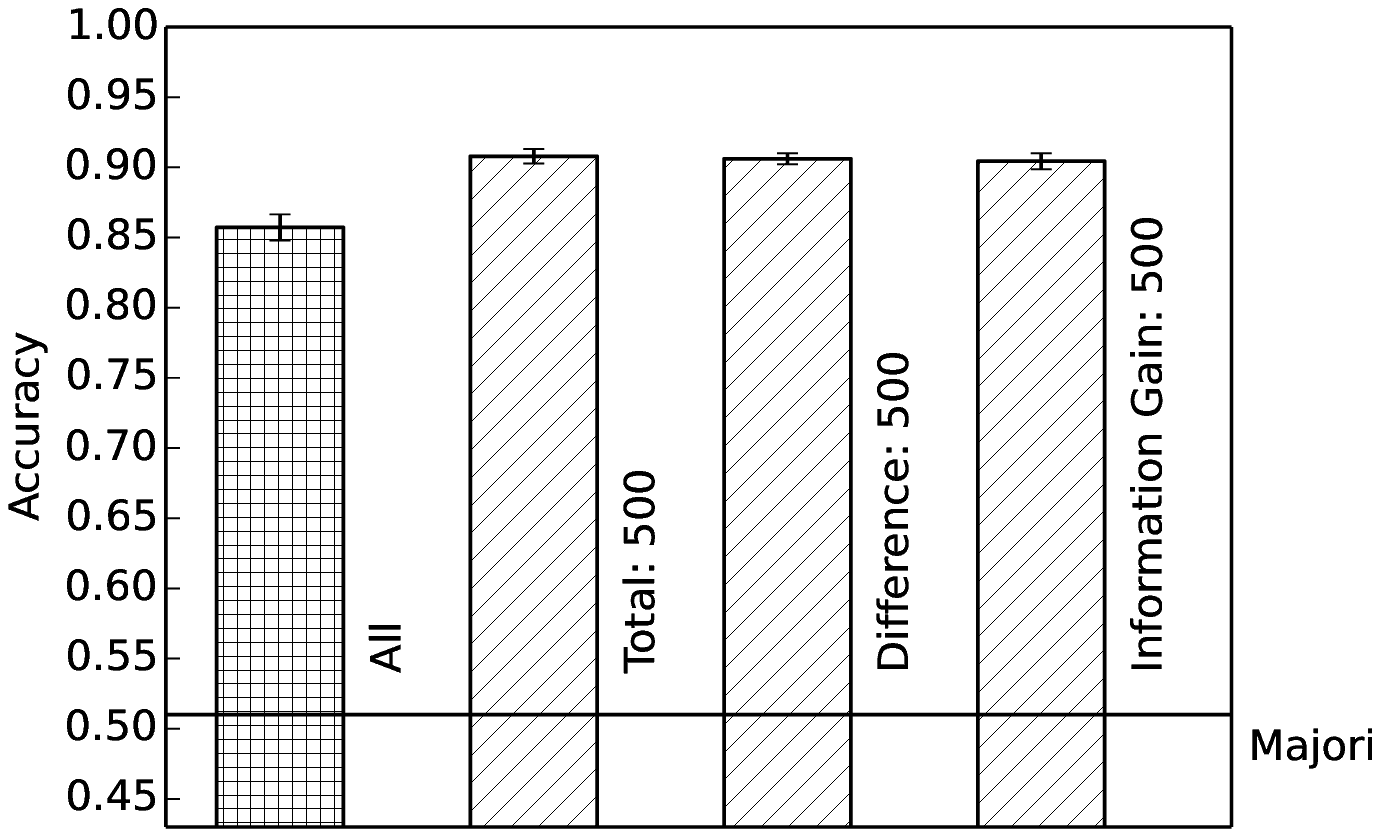}}
  \subfigure[Reuters Private]{\label{fig:reut_p}\includegraphics[width=0.69\textwidth]{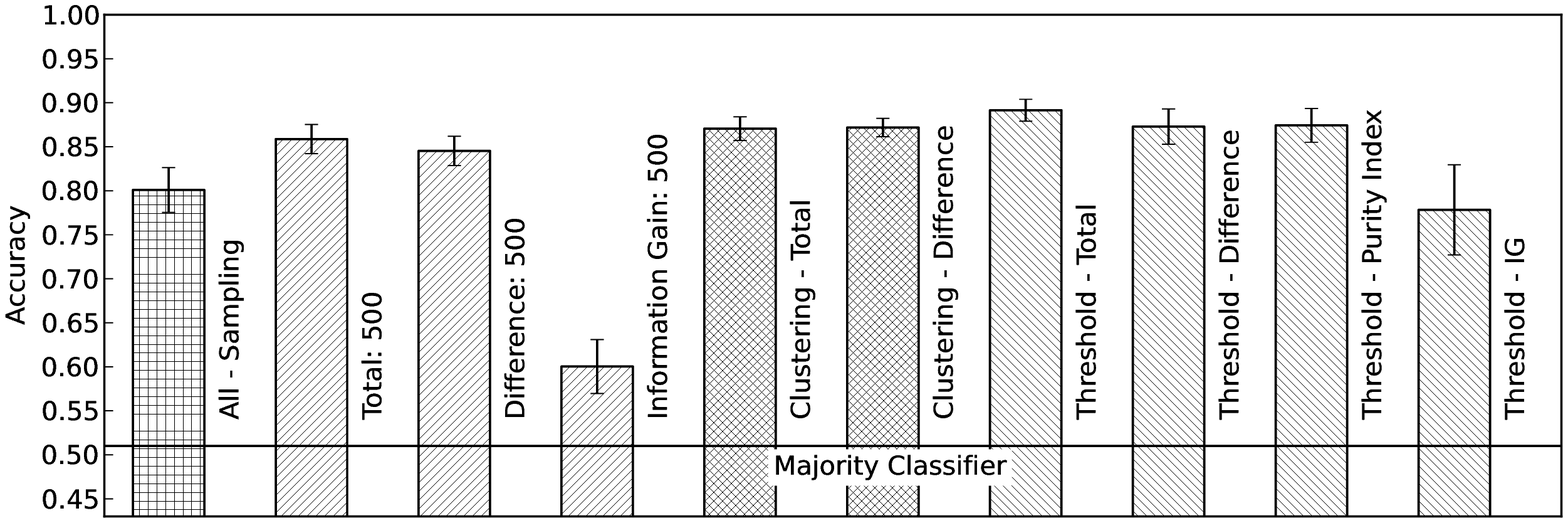}}
  \caption{\label{fig:all-wide}Naive Bayes classification with the \twitter, \sms and \reuters datasets.}
\end{figure*}

\begin{figure*}[t]
 %
  \subfigure[SMS Non-private]{\label{fig:sms_erm_np}\includegraphics[width=0.31\textwidth]{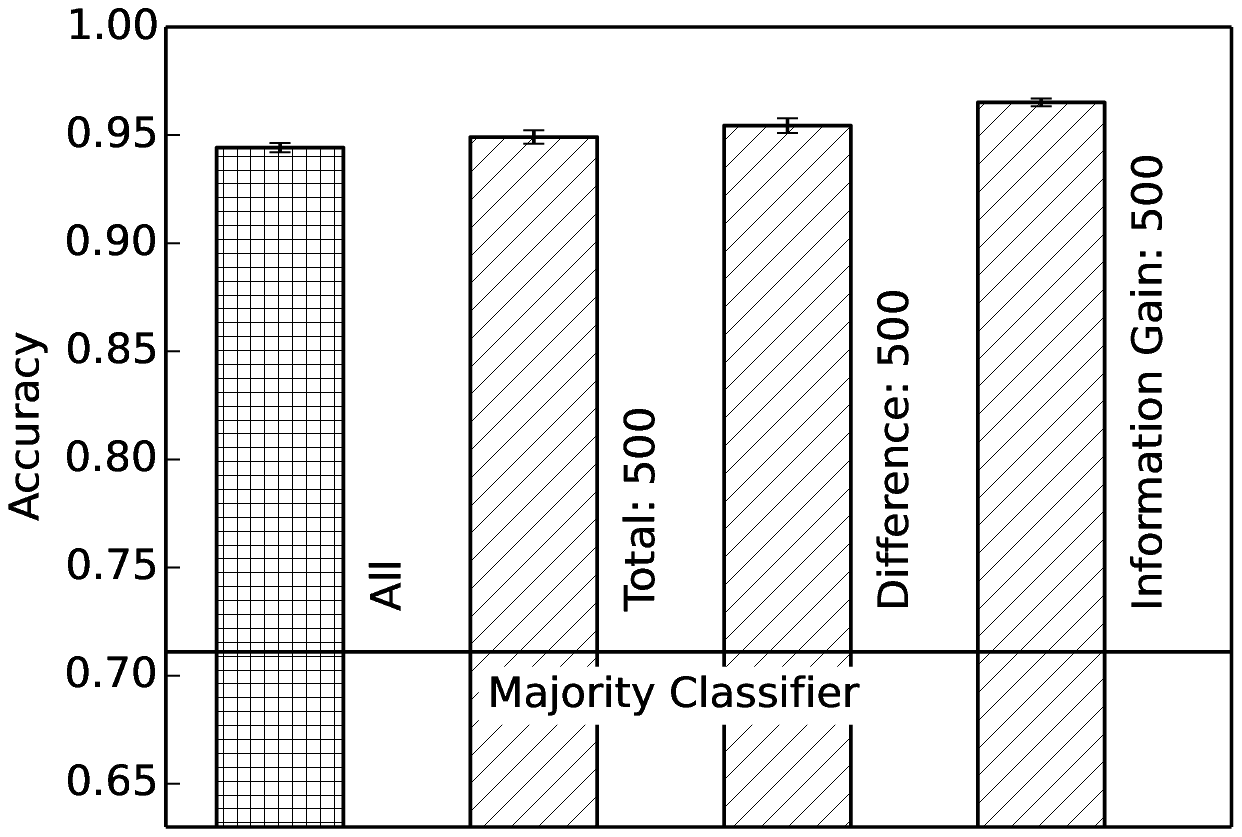}}
  \subfigure[SMS Private]{\label{fig:sms_erm_p}\includegraphics[width=0.69\textwidth]{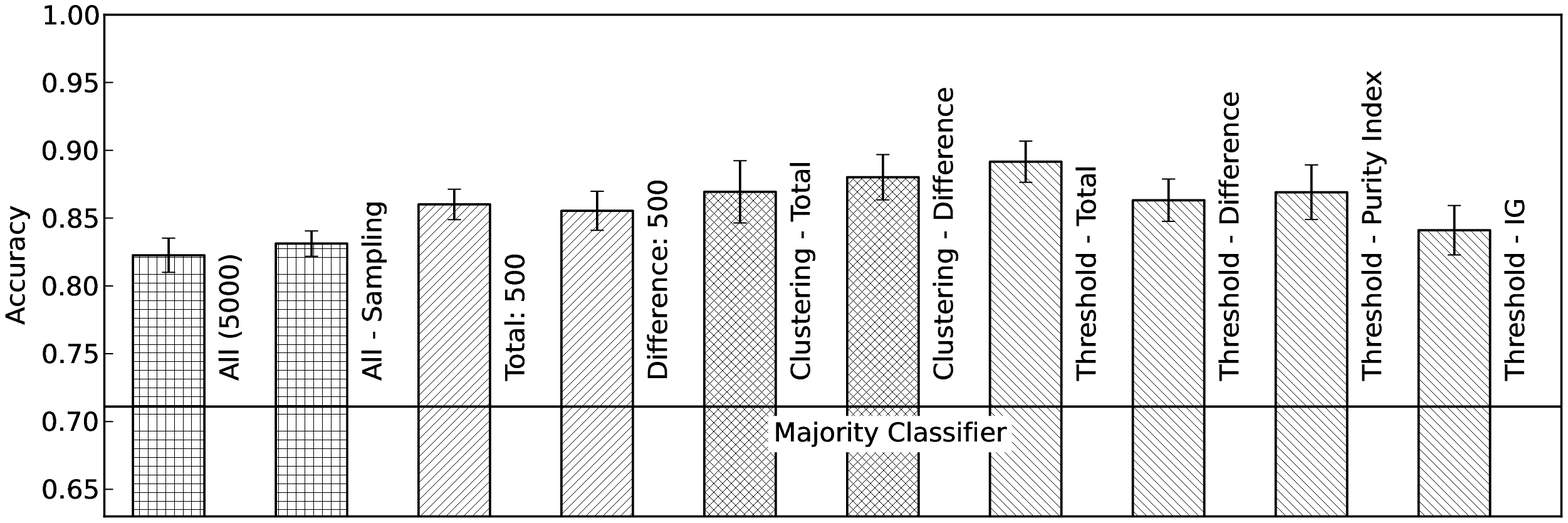}}
 %
 %
  \caption{\label{fig:logreg_wide}Logisitic regression with the \sms dataset. ERM used for private regression.}
\end{figure*}

Figure \ref{fig:all-wide} presents a comparison of all the discussed feature selection methods across all
three data sets using a non-private and a private naive bayes classifier. In the non-private case (Figures~\ref{fig:twit_np},\ref{fig:sms_np}, and \ref{fig:reut_np}), we see a small improvement in the accuracy using all three scoring techniques $TC$, $DC$ and $IG$. $PI$ resulted in a similar accuracy as $DC$ and is not shown.  $IG$ has the highest accuracy for all the datasets.

In the private case (Figures~\ref{fig:twit_p},\ref{fig:sms_p}, and \ref{fig:reut_p}), `All' corresponds to no feature selection, and `All-sampling' correponds to using all the features but with sampling (to reduce the sensitivity) with $r=10$. For the private graphs, the total $\epsilon$-budget is 1.0. We see that even though sampling throws away valuable data, we already see an increase in the accuracy. This is because sampling also helps reduce the sensitivity of the classifier training algorithm. Note that we do not report the `All' bar for \reuters -- since we can't bound the lenght of an article, the sufficient statistics for the naive bayes classifier have a very high sensitivity.  We also show the accuracy of the majority classifier, which always predicts the majority class.

Next we add feature selection. Both score perturbation and clustering are used in conjunction with sampling (to reduce sensitivity). Private threshold testing (PTT) does not use sampling. For score perturbation the budget split is .5 for selection and .5 for classification (budget split is discussed in Section~\ref{sec:other}). For clustering and PTT the budget split is .2 for selection and .8 for classification. 

We see that most of the feature selection techniques (and scoring functions) result in a higher accuracy than `All-sampling'. One exception is $IG$ due to its high sensitivity.  Additionally as noted in section \ref{sec:score_pert}, experiments with score perturbation 
of Information Gain  were run under bounded differential privacy (since the sensitivity of $IG$ is higher under unbounded differential privacy). We see poor accuracy with IG and score perturbation despite this. We do not report $IG$ under clustering and $PI$ under score perturbation and clustering due to their high sensitivity. We are surprised to see that $TC$ is as good as or better than ``best'' non-private scoring techniques across all three datasets and all differentially private feature selection techniques. This is due to its low sensitivity. We also note a trend that PTT with $TC$ is more accurate than clustering with $TC$ which is in turn more accurate than score perturbation with $TC$. 

Figure \ref{fig:logreg_wide} contains the same tests, but with the ERM classifier. We only show results on the \sms dataset due to space constraints.  We found that the private ERM code does not scale well to large number of features. 
For that reason we first selected the top 5000 features according to $TC$ scoring function and used that in place of the `All' features. Feature selection was then performed on this restricted dataset. We see the same trends as in the case of the Naive Bayes classifier.
The results are comparable to those run on the private Naive Bayes classifier, but with a lower
accuracy overall. This lower accuracy could be because Naive Bayes is known to outperform other methods for the  text classification task.

\begin{figure*}[t]
  \centering     
  \subfigure[PTT]{\label{fig:f1ptt}\includegraphics[width=0.24\textwidth]{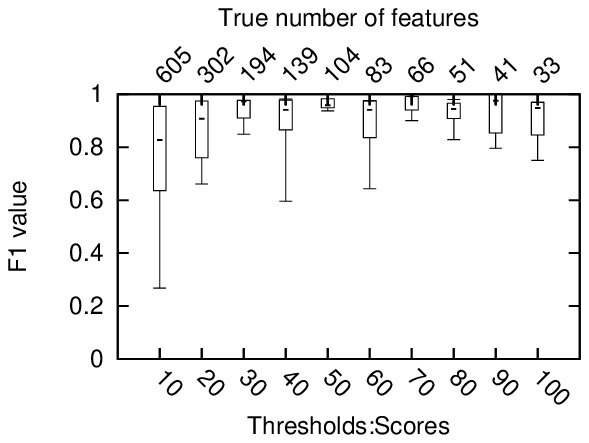}}
  \subfigure[NOISYCUT]{\label{fig:f1nc}\includegraphics[width=0.24\textwidth]{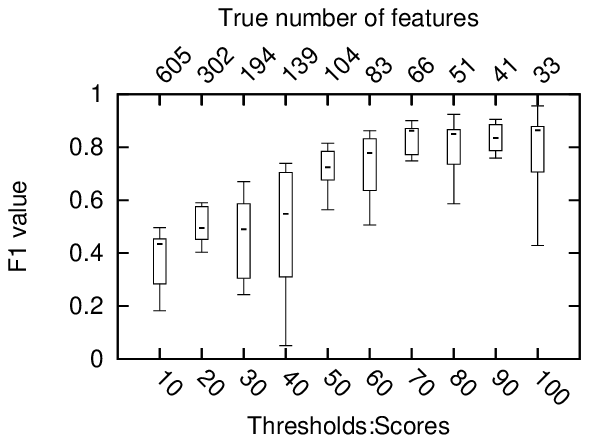}}
  \subfigure[Score Perturbation]{\label{fig:f1sp}\includegraphics[width=0.24\textwidth]{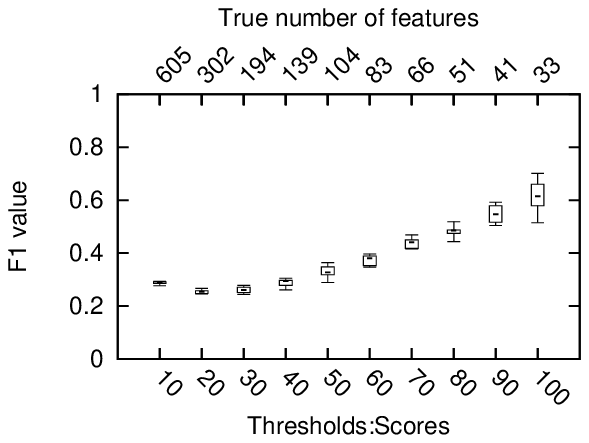}}
  \subfigure[Cluster]{\label{fig:f1cluster}\includegraphics[width=0.24\textwidth]{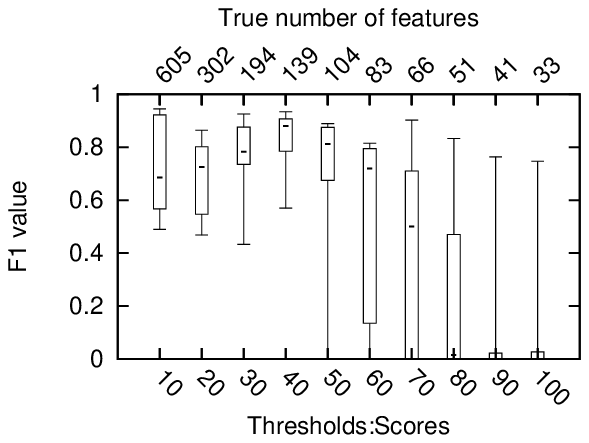}}
  \caption{F1 Score Comparison among 4 Private Feature Selection Algorithms on SMS data}
  \label{fig:scorebasedFS}
\end{figure*}

\subsubsection{{\sc ScoreBasedFS} Comparison}
We also evaluate the quality of the just the feature selection algorithms (without considering a classifier). The accuracy of a feature selection technique is quantified as follows. Let ${\cal F}_\tau$ be the true set of features whose scores are greater than the threshold (under some fixed scoring function), and let ${\cal F}'$ be the set of features returned by a differentially private algorithm for {\sc ScoreBasedFS}. We define precision (pre), recall (rec) and F1-score (F1) as follows: 
\[pre = \frac{|{\cal F}_\tau \cap {\cal F}'|}{|{\cal F}'|}, \ rec =\frac{|{\cal F}_\tau \cap {\cal F}'|}{|{\cal F}_\tau|}, \ F1 = \frac{2\cdot pre\cdot rec}{pre + rec}\]
Figure \ref{fig:scorebasedFS} shows the F1 scores for 4 private feature selection methods using $TC$  -- score perturbation, clustering, PTT and {\sc noisycut} \cite{lee14:kdd}. The x-axis corresponds to different thresholds $\tau$. The x-axis values on the top represent $|{\cal F}_\tau|$. 

There are two notable features of these plots. First, PTT does the best of all selection methods at all thresholds. This is due to the fact that only the threshold is perturbed. Since the ordering of feature scores is maintained, ${\cal F}'$ is a superset of ${\cal F}_\tau$ (with $rec = 1$) or is a subset of ${\cal F}_\tau$ (with $pre = 1$). In particular it significantly outperforms {\sc noisycut} under small thresholds (or when many features must be chosen). Second, we are able to see what
settings would cause the other methods to struggle. Both score perturbation and {\sc noisycut} have poorer accuracy as $\tau$ decreases (or number of features increases). This is because feature score are perturbed, and as we increase the number of features to be selected there is a larger chance that good features are eliminated and poorer features are returned just by random chance. Clustering shows the reverse trend. This is because low scoring features tend to cluster together resulting in large clusters (resulting in low sensitivity).  The same is not true for high scoring features.

\begin{figure*}[t]
  \centering     
  \subfigure[Twitter Budget]{\label{fig:twit_budget}\includegraphics[width=0.33\textwidth]{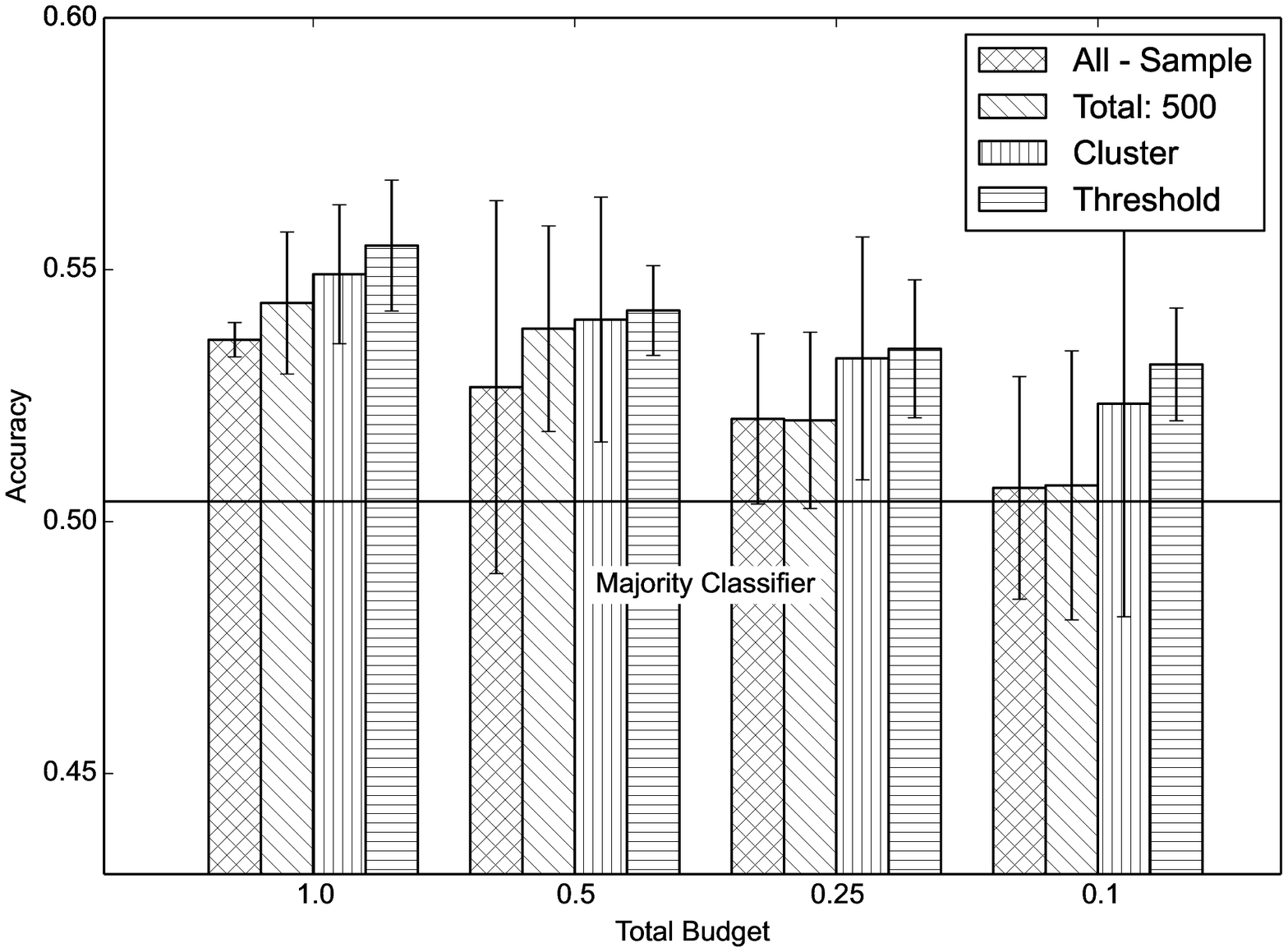}}
  \subfigure[SMS Budget]{\label{fig:sms_budget}\includegraphics[width=0.33\textwidth]{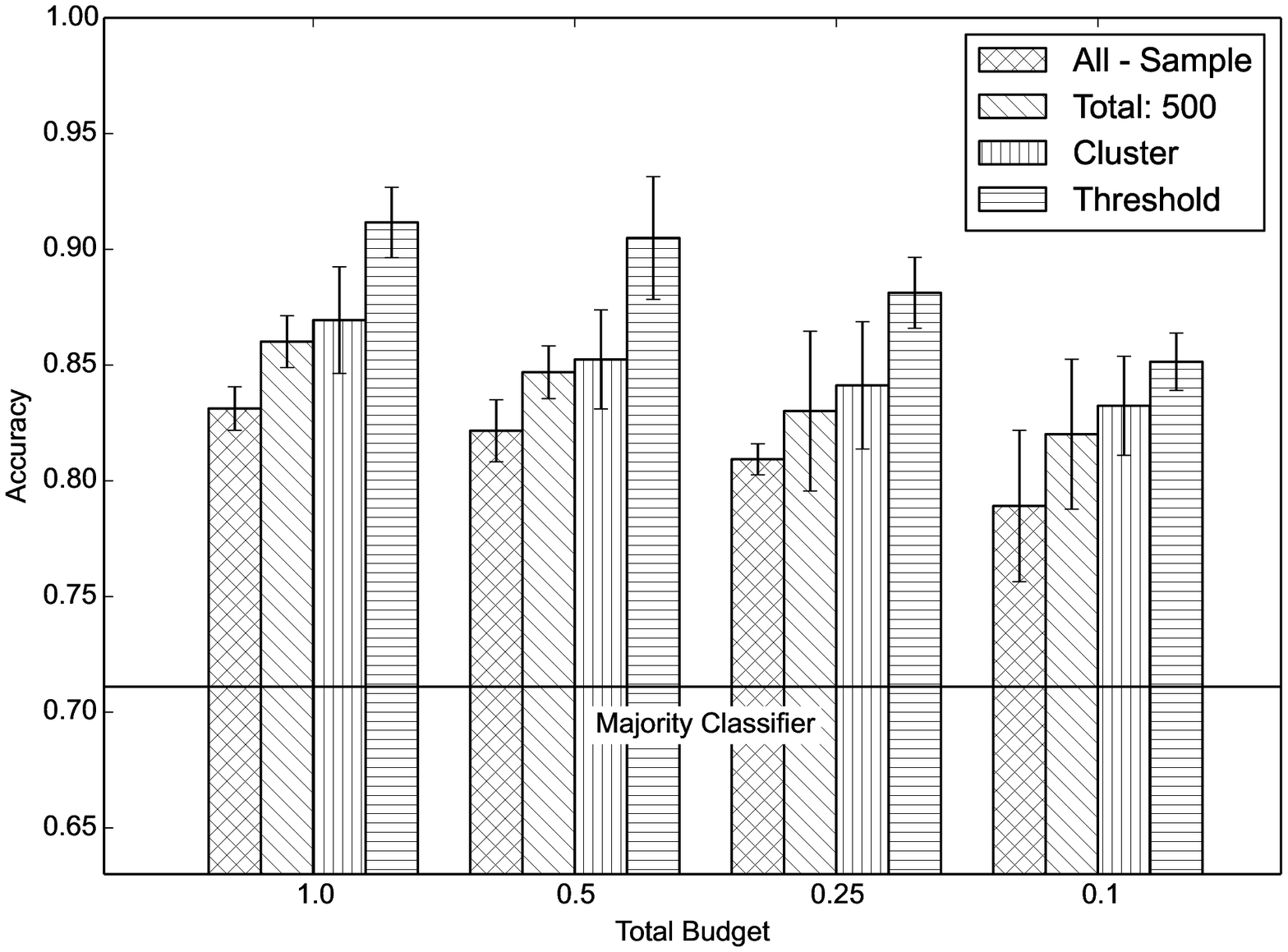}}
  \subfigure[Reuters Budget]{\label{fig:reut_budget}\includegraphics[width=0.33\textwidth]{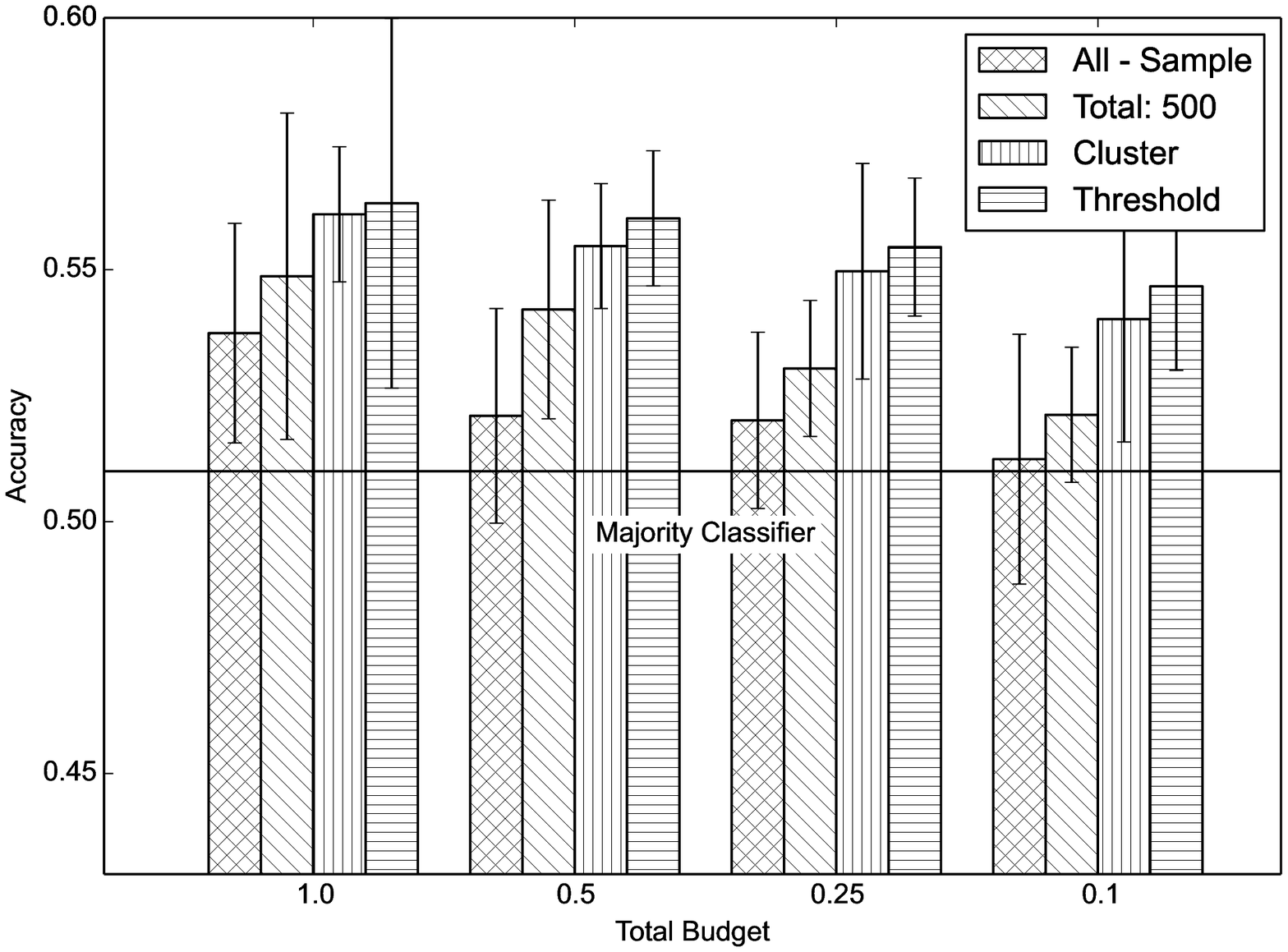}}
  \caption{Accuracy of the system versus total budget selection.}
  \label{fig:total-budget}
\end{figure*}

\begin{figure*}[t]
  \centering     
  \subfigure[Twitter Budget]{\label{fig:twit_split_budget}\includegraphics[width=0.24\textwidth]{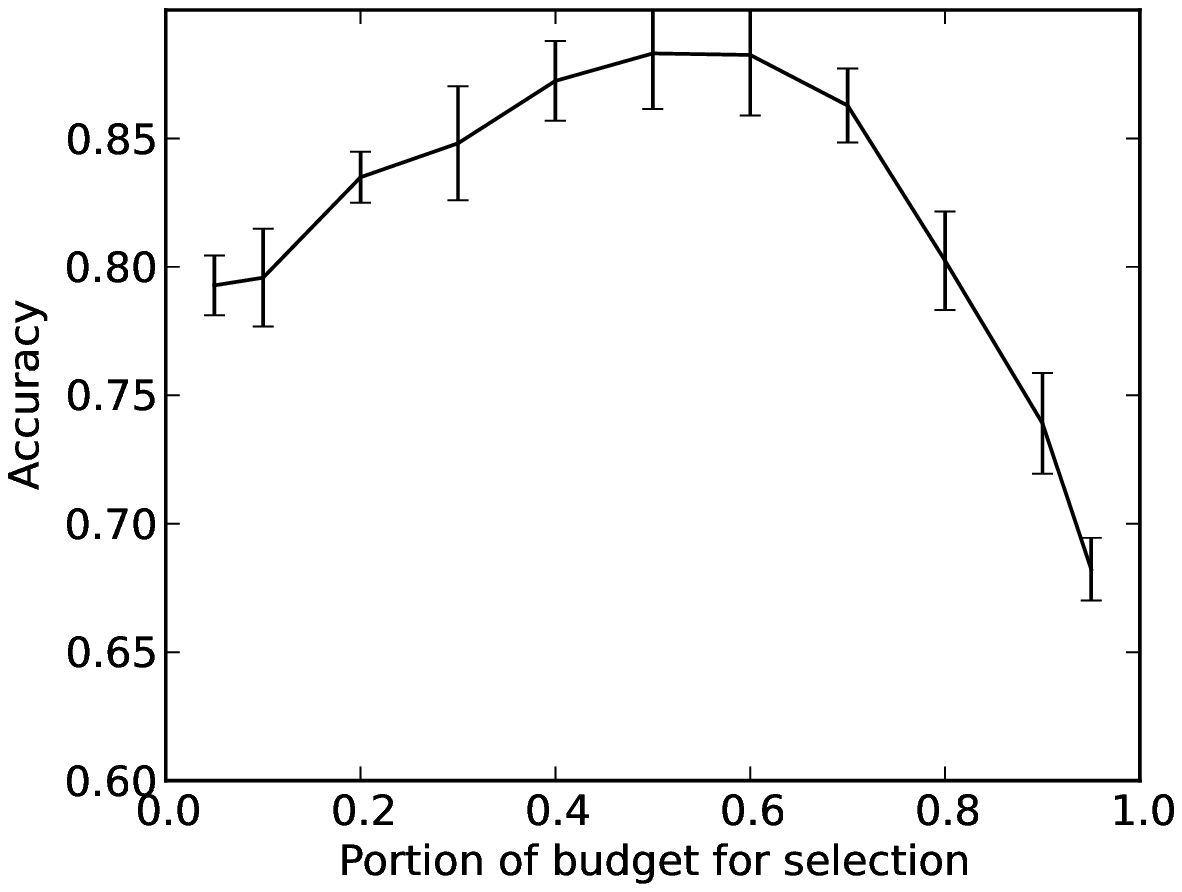}}
  \subfigure[SMS Budget]{\label{fig:sms_split_budget}\includegraphics[width=0.24\textwidth]{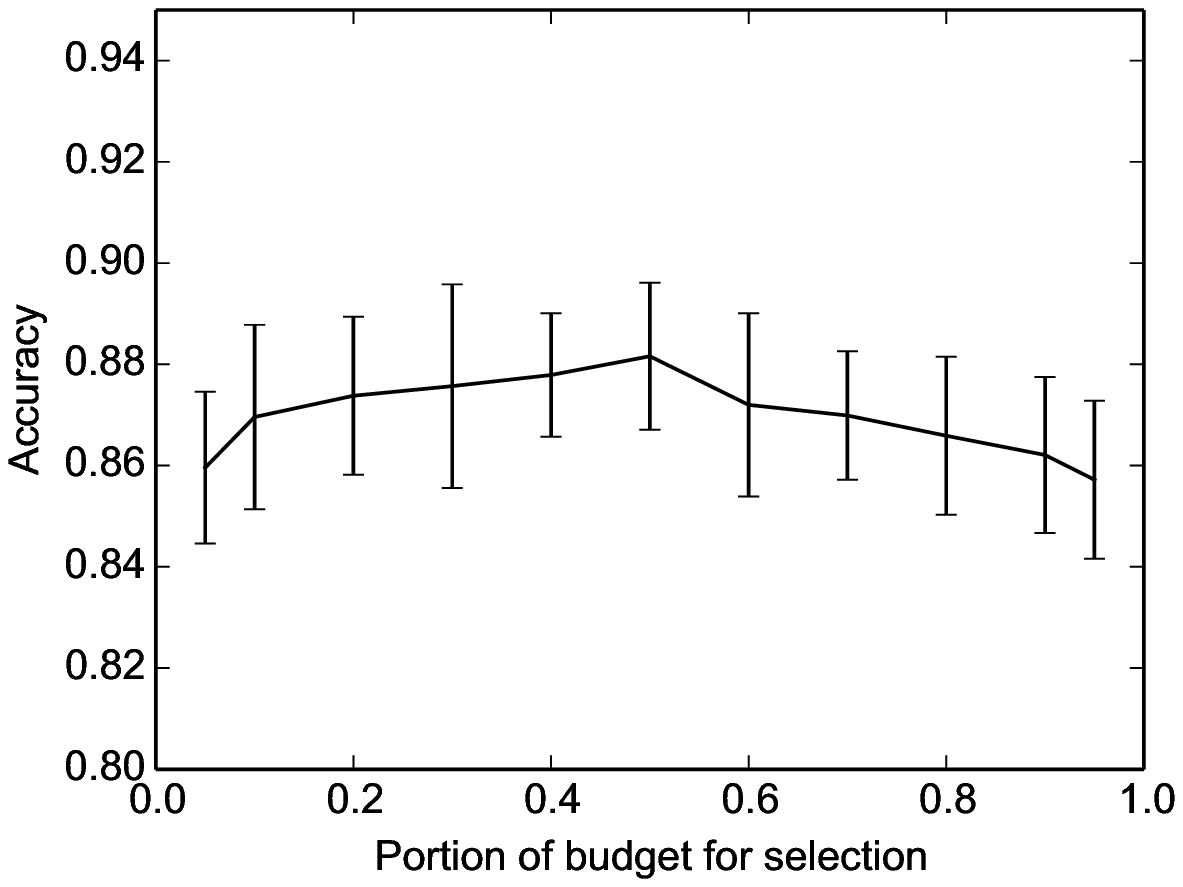}}
  \subfigure[Twitter Sampling Rate]{\label{fig:twit_samp}\includegraphics[width=0.24\textwidth]{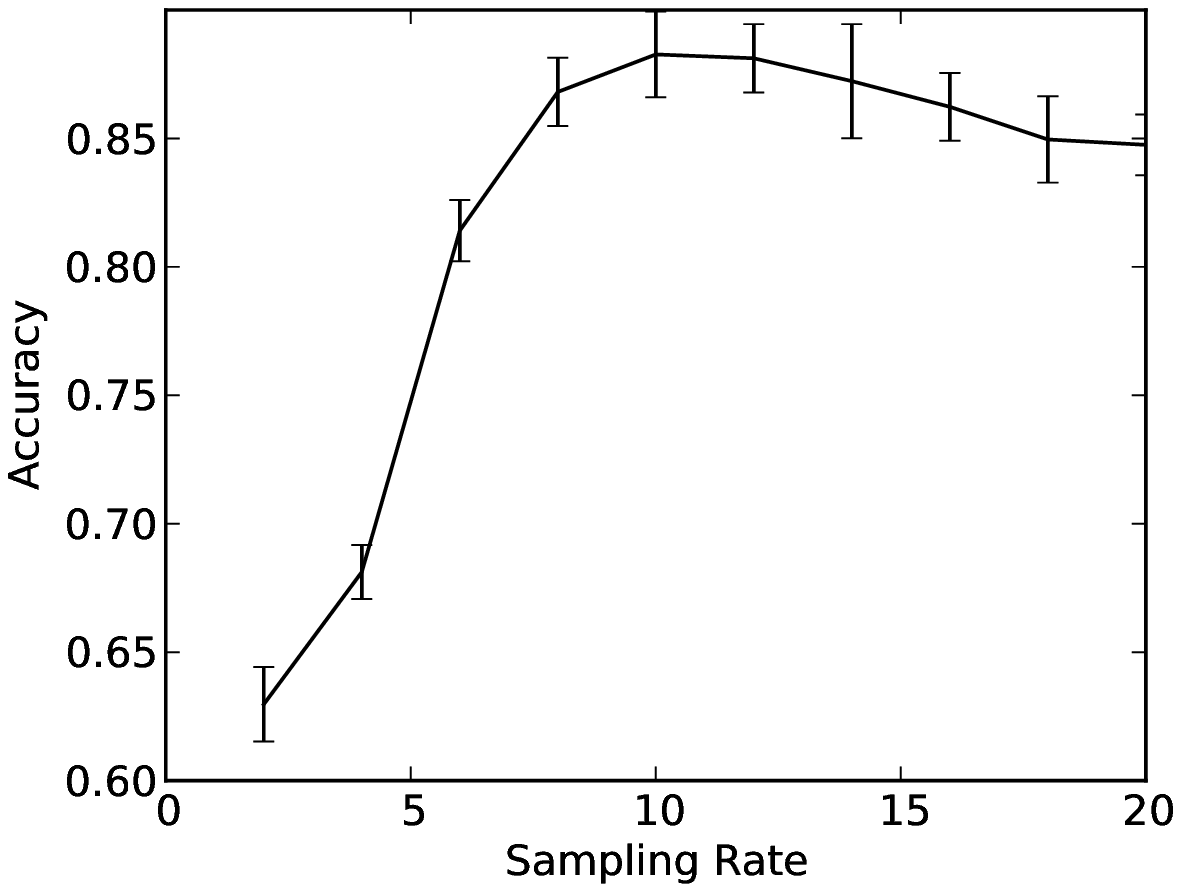}}
  \subfigure[SMS Sampling Rate]{\label{fig:sms_samp}\includegraphics[width=0.24\textwidth]{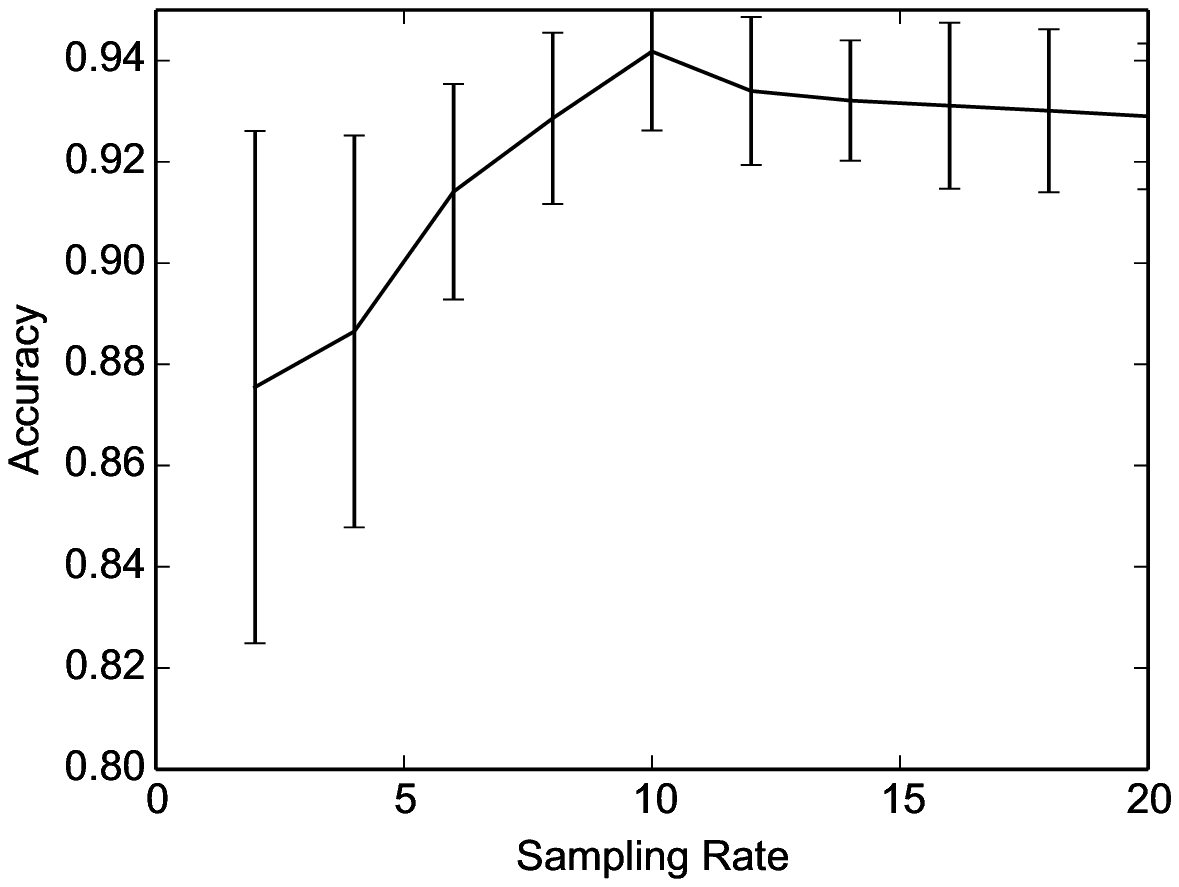}}
  \caption{Tuning the budget split (without sampling) and sampling rate (with a 50\% budget split).}
  \label{fig:tradeoffs}
\end{figure*}

\subsubsection{Parameter Tuning}
\label{sec:other}
In this section we present empirical justification for some of our design choices -- budget split, and sample rate selection. 
We defer the problem of classifier agnostic automatic parameter tuning to future work.

\noindent{\bf Privacy Budget Split:}
We empirically tested the accuracy of the classifier with feature selection under different budget splits. Figures \ref{fig:twit_split_budget} and \ref{fig:sms_split_budget} show (on the \twitter and \sms datasets, resp.) one example of Naive Bayes classification with score perturbation using the $TC$ score function. We see that the best accuracy is achieved when feature selection and classification equally split the budget. Since clustering and PTT have much lower sensitivities we find than a much smaller part of the budget (0.2) is required for these techniques to get the best accuracy (graphs not shown).

\noindent{\bf Splitting data vs Privacy Budget:}
Rather than splitting the privacy budget, one could execute feature selection and classification on disjoint subsets of the data. By parallel composition, one can use all the privacy budget for both tasks. However, experiments on the Naive Byes classifier with PTT showed splitting the data resulted in classifiers whose average accuracy was very close to that of the majority classifier. Since feature selection is run on a slightly different dataset,  wrong features are being chosen for classifier training. 

\noindent{\bf Sampling Rate Selection:}
Figures \ref{fig:twit_samp} and \ref{fig:sms_samp} show the change in system accuracy as the sampling rate $r$ is changed for the \sms and \twitter data sets. We see that selecting alow sampling rate is detrimental since too much information is lost. Alternatively selecting
a sampling rate that is too high loses accuracy from increasing the sensitivity used when drawing noise for privacy. A moderate rate of sampling that preserves enough information while reducing the required global sensitivity for privacy will do best.

\noindent{\bf Total Budget Selection}
Figure \ref{fig:total-budget} shows the accuracy of the private classifier with the best private score perturbation, clustering and PTT feature selection algorithm under different settings of the total privacy budget ($\epsilon = 1, 0.5, 0.25$ and $0.1$).
The same settings for budget split and sampling are held throughout. For the \twitter and \reuters datasets (which are harder to
predict) we see the accuracy begin to approach the majority classifier as the total budget is reduced to 0.1.

\subsection{Private Evaluation}\label{sec:rocexp}
We use the held out test sets of the \sms and \twitter datasets which come from the previous section. \sms test set contains 558 data, and each tuple $t$ has a true label $t[L] \in \{0,1\}$ as well as a prediction  $p(t) \in [0,1]$ for the label $1$. In \sms, 481 out of 558 data have true label equals $1$. The \twitter test dataset contains 684 different tuples, 385 of which have label equal to $1$.

Figure \ref{fig:base-roc} shows the real ROC curves as well as the differentially private ROC curves for both \sms and \twitter datasets under 4 different privacy budgets by using \alphamed. Recall that, in \alphamed,    $\epsilon_1$ privacy budget is used for selecting a set of thresholds, and $\epsilon_2 = \epsilon - \epsilon_1$ is used for generating the ROC curve. We set $\epsilon_{1} = 0.2\epsilon$.

\begin{figure*}[t]
\subfigure{
\begin{minipage}[t]{0.25\linewidth}{
\centering
\includegraphics[width = 1.8in]{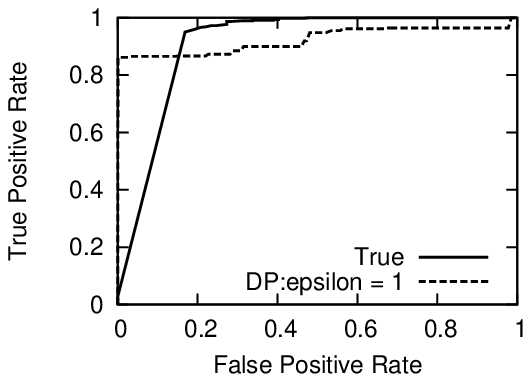}
}
\end{minipage}
\begin{minipage}[t]{0.25\linewidth}{
\centering
\includegraphics[width = 1.8in]{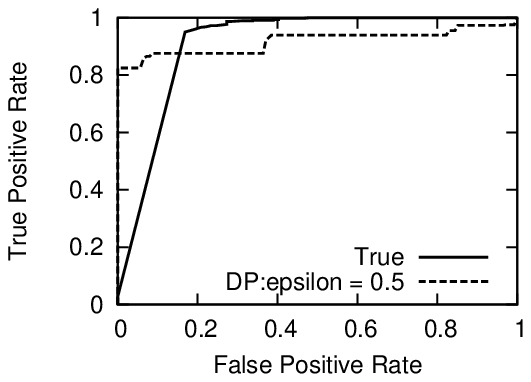}
}
\end{minipage}

\begin{minipage}[t]{0.25\linewidth}{
\centering
\includegraphics[width = 1.8in]{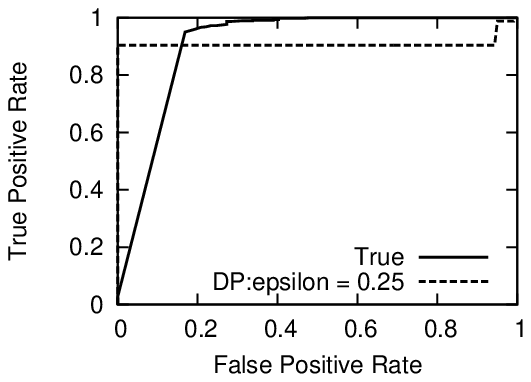}
}
\end{minipage}
\begin{minipage}[t]{0.25\linewidth}{
\centering
\includegraphics[width=1.8in]{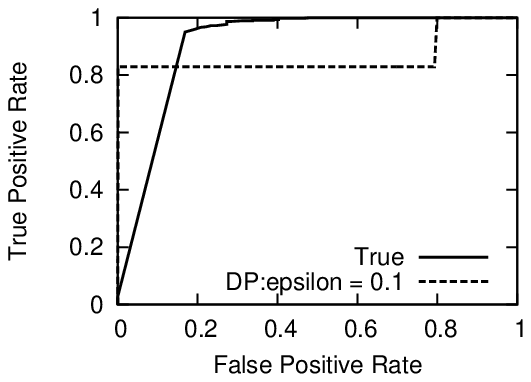}
}
\end{minipage}
}
\subfigure{
\begin{minipage}[t]{0.25\linewidth}{
\centering
\includegraphics[width = 1.8in]{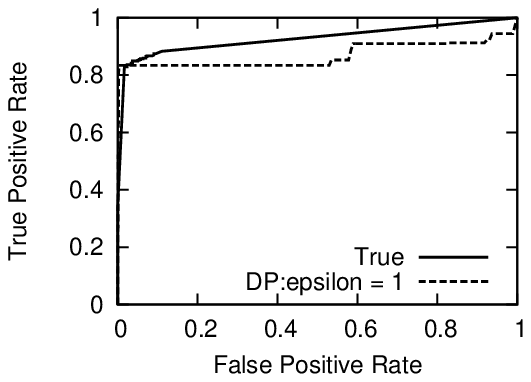}
}
\end{minipage}
\begin{minipage}[t]{0.25\linewidth}{
\centering
\includegraphics[width = 1.8in]{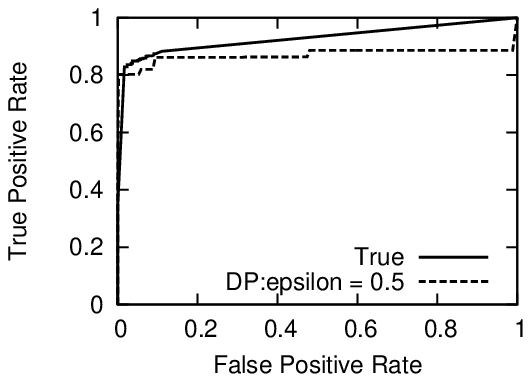}
}
\end{minipage}

\begin{minipage}[t]{0.25\linewidth}{
\centering
\includegraphics[width = 1.8in]{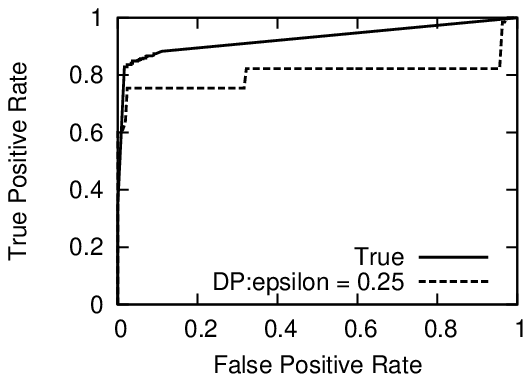}
}
\end{minipage}
\begin{minipage}[t]{0.25\linewidth}{
\centering
\includegraphics[width = 1.8in]{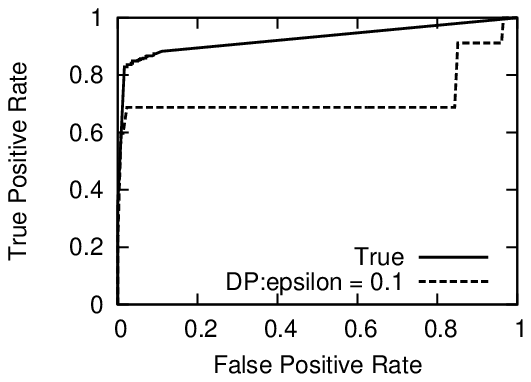}
}
\end{minipage}
}
\vspace{-5mm}
\caption{True \&  Private ROC curves, $\emph{SMS}$ (above) and $\emph{TWITTER}$ (below) datasets, $\epsilon = 1, 0.5, 0.25$ and $0.1$ (left to right)}
\label{fig:base-roc}
\end{figure*}

The solid line refers to the real ROC curve, while the dashed line represents the differentially private ROC curve. When the privacy budget is not small ($\epsilon = 1$), the private ROC curve is very close to the real ROC curve, which means our private ROC curve is a good replacement of the real ROC curve, correctly reflecting the performance of the input classifier.

\begin{table}
\centering
{\small
\begin{tabular}[t]{|c||c|c|c|c|c|c|}
\hline
$\epsilon$ & \multicolumn{6}{c|}{AUC Error}\\
\cline{2-7}
&\multicolumn{2}{c|}{Laplace} & \multicolumn{2}{c|}{\alphafs} & \multicolumn{2}{c|}{\alphamed} \\
\hline
& SMS &TWI & SMS&TWI &SMS & TWI  \\
\hline
1 &0.218 & 0.073 & 0.034 & 0.065 & 0.023 & 0.055  \\
\hline
0.5 & 0.343 & 0.108 & 0.042 & 0.094 & 0.029 & 0.063 \\
\hline
0.25 & 0.372 & 0.211 & 0.079 & 0.151 & 0.054 & 0.102 \\
\hline
0.1 & 0.442 & 0.340 & 0.146 & 0.229 & 0.092 & 0.203 \\
\hline
\end{tabular}
\caption{ROC area $L_{1}$ error (median) for both $\emph{SMS}$ and $\emph{TWITTER}$ datasets based on all possible thresholds}
\label{table:base-areaerror}
}
\end{table}

\begin{figure*}[t]
\subfigure{
\begin{minipage}[t]{0.25\linewidth}{
\centering
\includegraphics[width = 1.8in]{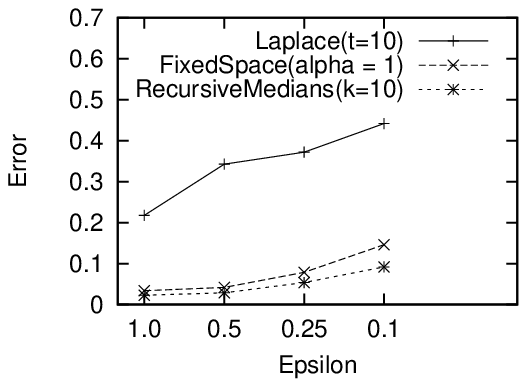}
}
\end{minipage}
\begin{minipage}[t]{0.25\linewidth}{
\centering
\includegraphics[width = 1.8in]{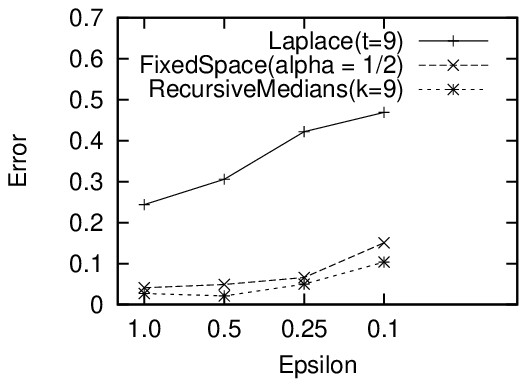}
}
\end{minipage}

\begin{minipage}[t]{0.25\linewidth}{
\centering
\includegraphics[width = 1.8in]{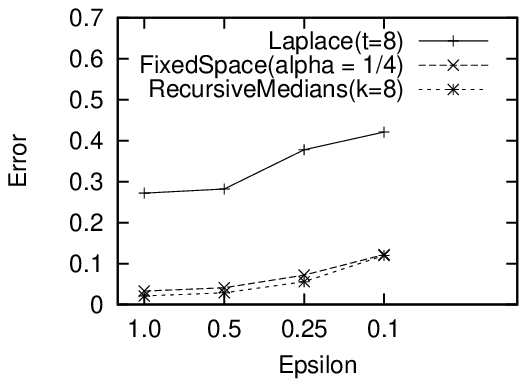}
}
\end{minipage}
\begin{minipage}[t]{0.25\linewidth}{
\centering
\includegraphics[width=1.8in]{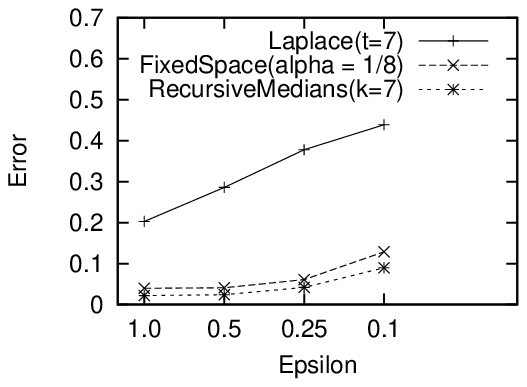}
}
\end{minipage}
}
\subfigure{
\begin{minipage}[t]{0.25\linewidth}{
\centering
\includegraphics[width = 1.8in]{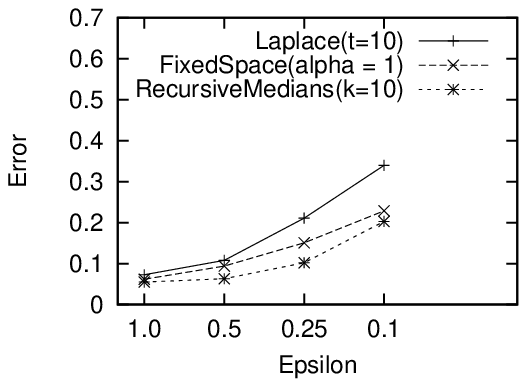}
}
\end{minipage}
\begin{minipage}[t]{0.25\linewidth}{
\centering
\includegraphics[width = 1.8in]{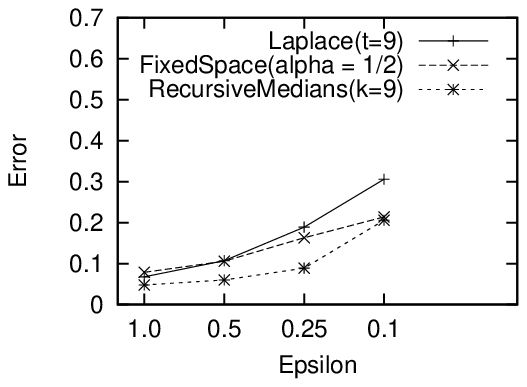}
}
\end{minipage}

\begin{minipage}[t]{0.25\linewidth}{
\centering
\includegraphics[width = 1.8in]{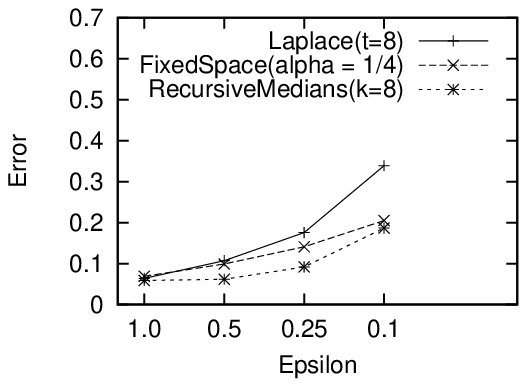}
}
\end{minipage}
\begin{minipage}[t]{0.25\linewidth}{
\centering
\includegraphics[width = 1.8in]{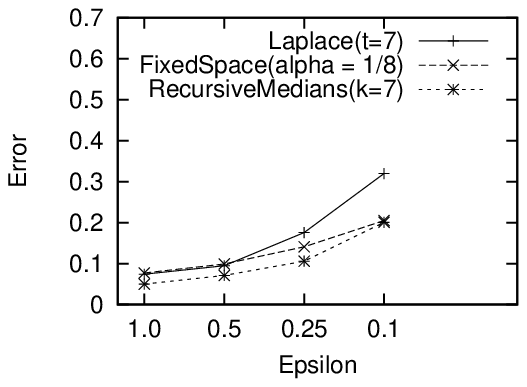}
}
\end{minipage}
}
\vspace{-5mm}
\caption{Error Comparison, $\emph{SMS}$ (above) and $\emph{TWITTER}$ (below) datasets}
\label{fig:area-comparison}
\end{figure*}

Figure \ref{fig:area-comparison} reports the comparison of the errors among three different algorithms. The $\emph{Laplace}$ line refers to the error by directly using $\emph{Laplace Mechanism}$ using $t$ thresholds (that can be chosen based on the data). The $\emph{FixedSpace}$ line shows the error of \alphafs, with $\alpha \cdot n$ thresholds chosen uniformly in $[0,1]$. And the line $\emph{RecursiveMedians}$ presents the error of \alphamed, with $2^k$ thresholds. The x-axis corresponds to $\epsilon$ and the y-axis corresponds to the L1 error of the area between the real ROC curve and the private ROC curve under certain privacy budget. The error shows the median value after running our algorithm 10 times. The reason why we pick the median error instead of the average value is to counter the effect of outliers.

To understand the effect of choosing differing numbers of thresholds, we choose $\alpha = 1, 0.5, 0.25$ and $0.125$. The ensure that the noise introduced is roughly the same in all algorithms, we vary $t$ and $k$ as $10, 9, 8$ and $7$. For instance, for $\alpha =1$ and $t, k = 10$, we have $O(n)$ thresholds for \alphafs and \alphamed ($2^9 \leq n \leq 2^{10}$ for the \sms and \twitter datasets), and $O(\log n)$ thresholds for Laplace.

In figure \ref{fig:area-comparison}, we can see that \alphamed and \alphafs can largely improve the accuracy of the output compared with directly using $\emph{Laplace Mechanism}$ under all $\epsilon$ and $\alpha$ settings. The difference in error is largest for small epsilon. Although the noise scale is the same for the three methods in each experiment, since the number of thresholds is very small the $\emph{LaplaceROC}$ curve can't hope to approximate the true ROC curve well enough. Furthermore, for both datasets, the \alphamed method performs better than \alphafs method under nearly all parameter settings, which means computing noisy quantiles help choose the right set of thresholds. 

Table~\ref{table:base-areaerror} represents the graphs in Figure~\ref{fig:area-comparison} for $t, k = 10$ and $\alpha = 1$ in tabular form. It is interesting to note that \alphamed has 10 times lower error than Laplace for $\epsilon = 1$ on \sms.

\subsubsection{Choosing the number of thresholds}
The value of $\alpha$ in \alphafs and $k$ in \alphamed determines the number of thresholds we will use to compute ROC curve. It will affect three different aspects. First, the bigger size of thresholds, the better we ca hope to approximate the true ROC curve. Second, larger threshold sets result in larger noise being used to perturb $TPRs$ and $FPRs$. Third, the last step of our algorithm is to do postprocessings in order to maintain consistency and its relationship to $\alpha, k$ is not very clear. Our goal is to pick the value of $\alpha, k$ which lead to the best trade off.

\begin{figure}[t]
\subfigure{
\centering
\includegraphics[width = 1.8in]{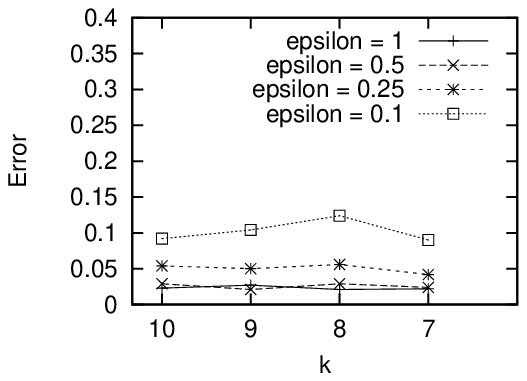}
\centering
\includegraphics[width = 1.8in]{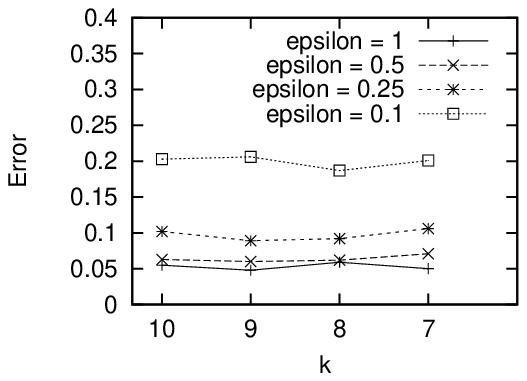}

}
\vspace{-5mm}
\caption{Number of thresholds: $\emph{SMS}$ (left) and $\emph{TWITTER}$ (right) datasets, based on all $\epsilon$}
\label{fig:alpha_comparison}
\end{figure}
Figure \ref{fig:alpha_comparison} presents the comparions of the errors for \alphamed among different choies for $k$ under all $\epsilon$ values. The x-axis show 4 different settings for $k$. We can see that for both datasets, there is no specific setting of $k$ that leads to the best performance of \alphamed  for all $\epsilon$ settings. The graph looks similar for \alphafs (not shown).

Thus, it seems best to set $k = \lceil \log n\rceil$. The following is one possible reason for the AUC error not depending on $k$: One may pick a small number of thresholds to reduce the noise if the true ROC curve can be accurately described using a small number of points. But in this case, the postprocessing step that enforces monotonicity results in error that has a strong dependence on the number of distinct $TPR$ and $FPR$ values, and not the total number of thresholds (see Theorem 2 \cite{boostacc_consistency}). 

\section{Related Work}
\label{sec:related}
\noindent{\bf Differentially Private Classifiers:}
Private models for classification has been a popular area of exploration
for privacy research. Previous work has produced differentially private training algorithms for  
Naive Bayes classification \cite{Vaidya2013},  decision trees \cite{Friedman2010, Jagannathan2012}, logistic regression \cite{Sarwate, zhang13privgene} and support vector machines \cite{Sarwate} amongst others.
Apart from classifier training, Chaudhuri el al. \cite{chaudhuri13:nips} 
present a generic algorithm for differentially private parameter tuning and model selection. However, this work does not assume a blackbox classifier, and makes strong stability assumptions about the training  algorithm. In contrast, our algorithms are classifier agnostic. Additionally 
Thakurtha et al. \cite{thakurta2013differentially} present an algorithm for model selection again assuming strong stability assumptions about the model training algorithm. We would like to note that the work in these paper is in some sense orthogonal to the feature selection algorithms we present, and can be used in conjunction with the results in the paper (for instance, to choose the right threshold $\tau$ or the right number of features to select). 

\noindent{\bf Private Threshold Testing:}
As mentioned before, private threshold testing (PTT) is inspired by the sparse vector technique (SVT) \cite{hardt-svt} which was first used in the context of the multiplicative weights mechanism \cite{Hardt10amultiplicative}. While PTT aims to only release whether or not a query answer is greater than a threshold, SVT releases the actual answers that are above the threshold and thus can only release a constant number of answers. Lee et al \cite{lee14:kdd} solve the same problem as PTT in the context of frequent itemset mining. The  propose an algorithm {\sc noisycut} which we show is inferior to PTT. While the techniques for proving the privacy of all these techniques are similar, our proof for PTT is the tightest thus allowing us only add noise to the threshold and get the best utility (amongst competitors) for answering comparison queries.

\noindent{\bf Private Evaluation:}
Receiver operating characteristic (ROC) curves are used quantifying the prediction accuracy of binary classifiers. However, directly releasing the ROC curve may reveal the sensitive information of the input dataset \cite{privateROC}.
In this paper, we propose the first differentially private algorithm for generating private ROC curves under differential privacy.  Chaudhuri et al \cite{chaudhuri13:nips} proposes a generic technique for evaluating a classifier on a private test set. However, they assume that the global sensitivity of the evaluation algorithm is low.  Hence, their work will not apply to generating ROC curves, since the sufficient statistics for generating the ROC curve (the set of true and false positive counts) have a high global sensitivity. Despite this high sensitivity, we present strategies that can privately compute ROC curves with very low noise by modeling the sufficient statistics as one-sided range queries.

\section{Conclusions}\label{sec:conc}
In this paper, we presented algorithms that can aid the adoption of differentially private methods for classifier training on private data.  We present novel algorithms for private feature selection and experimentally show using three real high dimensional datasets that spending a part of the privacy budget for feature selection can improve the prediction accuracy of the classifier trained on the selected features. Moreover, we also solve the problem of privately generating ROC curves. This allows a user to quantify the prediction accuracy of a binary classifier on a private test dataset. In conjunction, these algorithms can now allow a data analyst to mimic typical `big-data' workflows that (a) preprocess the data (i.e., select features), (b) build a model (i.e., train a classifier), and (c) evaluate the model on a held out test set (i.e., generate an ROC curve) on private data while ensuring differential privacy without sacrificing too much accuracy.

\balance

{\small
\bibliographystyle{abbrv}
\bibliography{refs}
}

\clearpage
\appendix
\section{Proof of Theorem 2}\label{sec:PTTproof}
\begin{proof}
For any two neighboring datasets $D$ and $D'$, we would like to show:
\begin{eqnarray*}
\frac{P(v_D \rightarrow \hat{v})}{P(v_{D'} \rightarrow \hat{v})} &\leq &  e^{2\sigma \epsilon}
\end{eqnarray*}
where $v_D$ and $v_{D'}$ denote the outputs of the non private threshold test on $D$ and $D'$ resp., and $\hat{v}$ is the output of PTT.  Let $N^1 = \{ i\in [m]~ | \hat{v}[i] = 1\}$ and $N^0 = \{ i\in [m] ~| \hat{v}[i] = 0\}$ denote the set of $1$ and $0$ answers resp. of PTT. Let $\hat{v}[<i]$ denote the answers returned by PTT for queries $1$ through $i-1$. Then
\begin{eqnarray*}
\lefteqn{\frac{P(v_D \rightarrow \hat{v})}{P(v_{D'} \rightarrow \hat{v})}
\ = \ \prod_{i \in [m]} \frac{P(Q_i(D) = \hat{v}[i] \ | \ \hat{v}[<i])}{P(Q_i(D') = \hat{v}[i] \ | \  \hat{v}[<i])} }\\
&= & \prod_{i\in N^1} \frac{P(Q_i(D) =  1 \ |\  \hat{v}[<i])}{P(Q_i(D') = 1 \ |\ \hat{v}[<i])} 
\times \prod_{i\in N^0} \frac{P(Q_i(D) = 0 \ | \ \hat{v}[<i])}{P(Q_i(D') = 0 \ |\  \hat{v}[<i])}  \\
\lefteqn{\prod_{i\in N^1} P(Q_i(D) =  1 \ |\  \hat{v}[<i])}\\
& = & \int_z P(\tilde{\tau} = z) \prod_{i \in N^1}P(Q_i(D) =  1 \ |\ \tilde{\tau} = z) dz\\
& = & \int_z P(\tilde{\tau} = z) \prod_{i \in N^1}P(Q_i(D) > z) dz
\end{eqnarray*}
The following two facts complete the proof. First, for any $z$, 
\begin{equation}
P(\tilde{\tau} = z)  \ \leq \ e^{\sigma\epsilon} P(\tilde{\tau}  = z - \sigma)
\end{equation}
Second, for neighboring databases $D$ and $D'$, 
\begin{eqnarray}
\nonumber
Q_i(D) \geq z  & \Rightarrow &  Qi(D') \geq z - \sigma\\
	P(Q_i(D)  \geq z)  & \leq & P(Q_i(D')  \geq z - \sigma)
\end{eqnarray}
Therefore, we get: 
\begin{eqnarray*}
\lefteqn{\prod_{i\in N^1} P(Q_i(D) =  1 \ |\  \hat{v}[<i])}\\
& = & \int_z P(\tilde{\tau} = z) \prod_{i \in N^1}P(Q_i(D) > z) dz\\
& \leq & e^{\sigma\epsilon}\int_z P(\tilde{\tau} = z-\sigma) \prod_{i \in N^1}P(Q_i(D') > z-\sigma) dz\\
& = & e^{\sigma\epsilon}\prod_{i\in N^1} P(Q_i(D') =  1 \ |\  \hat{v}[<i])
\end{eqnarray*}
An analogous proof for $N^0$ gives the required bound of $e^{2\sigma\epsilon}$.
\end{proof}

\eat{\section{Proof of Corollary 2}\label{sec:PTTcorproof}
\begin{proof}(sketch)
{\em Case (i) $\tau$ is a constant:} When $D = D' \cup \{t\}$, for all $z$, $Q_i(D) < z$ implies $Q_i(D') < z$. Thus, $r_0 = \prod_{i\in N^0} \frac{P(Q_i(D) = 0 \ |\  \hat{v}[<i])}{P(Q_i(D') = 0 \ |\ \hat{v}[<i])}$ is already bounded above by $1$, while $r_1 = \prod_{i\in N^1} \frac{P(Q_i(D) = 1 \ | \ \hat{v}[<i])}{P(Q_i(D') = 1 \ |\  \hat{v}[<i])}$ is bounded by $e^{\sigma\epsilon}$ from proof of Thoerem~\ref{thm:PTT}. When $D' = D \cup \{t\}$, we have $r_1 < 1$ and $r_0 \leq e^{\sigma\epsilon}$. 

{\em Case (ii) $\tau$ is a function of $D$:} When $D = D' \cup \{t\}$, it still holds that $P(\tilde{\tau(D)} = z)  \ \leq \ e^{\sigma\epsilon} P(\tilde{\tau(D')}  = z - \sigma)$. This is because $\tau(D')$ lies between $[\tau(D)-\sigma, \tau(D)]$. The rest of the proof remains the same. 
\end{proof}
}

\end{document}